\documentclass[twoside]{article}

\usepackage[accepted]{aistats2023}

\usepackage[round]{natbib}

\usepackage[utf8]{inputenc} 
\usepackage[T1]{fontenc}    
\usepackage{xr-hyper}
\usepackage{hyperref}       
\usepackage{url}            
\usepackage{booktabs}       
\usepackage{amsfonts}       
\usepackage{nicefrac}       
\usepackage{microtype}      
\usepackage{xcolor}         
\usepackage[shortlabels]{enumitem}
\usepackage{algorithm}
\usepackage{algpseudocode}
\usepackage{cancel}
\usepackage{multirow}
\usepackage{graphicx}
\usepackage{subcaption}
\usepackage{pbox}
\DeclareCaptionLabelFormat{andtable}{#1~#2  \&  \tablename~\thetable}

\usepackage{amsmath}
\usepackage{amsthm}
\usepackage{amssymb}
\usepackage{shortcuts}

\makeatletter
\newcommand*{\addFileDependency}[1]{
  \typeout{(#1)}
  \@addtofilelist{#1}
  \IfFileExists{#1}{}{\typeout{No file #1.}}
}
\makeatother

\newcommand{\Yobs}{Y}
\newcommand{\Ypot}{Y}
\newcommand{\Xpi}{\phi(X)}
\newcommand{\kcdte}{\operatorname{CDTE}}

\theoremstyle{plain}
\newtheorem{theorem}{Theorem}

\newtheorem{corollary}[theorem]{Corollary}
\theoremstyle{definition}
\newtheorem{assumption}{Assumption}
\newtheorem{definition}{Definition}
\newtheorem{remark}{Remark}

\usepackage[capitalize]{cleveref}
\Crefname{assumption}{Assumption}{Assumptions}

\usepackage[compact]{titlesec}  

\begin{document}

\twocolumn[

\aistatstitle{Robust and Agnostic Learning of Conditional Distributional Treatment Effects}

\aistatsauthor{ Nathan Kallus \And Miruna Oprescu}

\aistatsaddress{Cornell University and  Cornell Tech } ]

\begin{abstract}

 The conditional average treatment effect (CATE) is the best measure of individual causal effects given baseline covariates.  
 However, the CATE only captures the (conditional) average, and can overlook risks and tail events, which are important to treatment choice.
 In aggregate analyses, this is usually addressed by measuring the distributional treatment effect (DTE), such as differences in quantiles or tail expectations between treatment groups. 
 Hypothetically, one can similarly fit conditional quantile regressions in each treatment group and take their difference, but this would not be robust to misspecification or provide agnostic best-in-class predictions. 
 We provide a new robust and model-agnostic methodology for learning the conditional DTE (CDTE) for a class of problems that includes conditional quantile treatment effects, conditional super-quantile treatment effects, and conditional treatment effects on coherent risk measures given by $f$-divergences. Our method is based on constructing a special pseudo-outcome and regressing it on covariates using any regression learner. Our method is model-agnostic in 
 that it can provide the best projection of CDTE onto the regression model class. 
 Our method is robust in that even if we learn these nuisances nonparametrically at very slow rates, we can still learn CDTEs at rates that depend on the class complexity and even conduct inferences on linear projections of CDTEs. We investigate the behavior of our proposal in simulations, as well as in a case study of 401(k) eligibility effects on wealth.

\end{abstract}

\section{INTRODUCTION}

Measuring treatment-effect heterogeneity along observed covariates is an important tool for interpreting the results of A/B tests on online platforms, program evaluations in social science whether experimental or observational, and clinical trials in medicine.
These analyses can help diagnose how the treatment works, understand for whom it does and does not work, and assess fairness \citep{heckman1997making, crump2008nonparametric, kent2010assessing, kallus2022treatment}.
On the other hand, measuring distributional treatment effects (DTEs) such as quantile treatment effects (QTEs) is an important tool for understanding the impact of interventions beyond the mean, especially when outcomes are naturally very skewed, like income or platform usage \citep{bitler2006mean, firpo, belloni, ldml}.

Motivated by the need to assess \emph{both} heterogeneity \emph{and} distributional impact, in this paper we study flexible, agnostic, and robust machine learning tools to estimate \emph{conditional} DTE (CDTE) functions.
Recent advances in causal machine learning have offered new methods for assessing effect heterogeneity by learning conditional \emph{average} treatment effects (CATEs) \citep{slearner,causaltree, causalforest, kunzel2019metalearners,kennedy2020optimal,nie2021quasi}. These works have highlighted the importance of learning CATEs \emph{directly}, rather than learning conditional-average outcomes by treatment arm and taking their difference (also known as the plug-in approach). One issue with the plug-in approach is that it can wash out the effect signal (not robust): \eg, many variables strongly predict baselines but only a few modulate the effect.
Another issue is that it fails to give best-in-class predictions (not agnostic): \eg, taking the difference of the best linear predictions of outcome by arm does not yield the best linear prediction of treatment effect.

We tackle the same challenges for CDTEs. We consider CDTEs for a very rich class of distributional metrics that includes quantiles, super-quantiles, and other coherent risk measures. Given any distributional metric in our class, we construct a pseudo-outcome that combines an initial guess for the CDTE along with a debiasing term. Our algorithm is then to regress this pseudo-outcome on covariates, using any given blackbox learner. In the case of the CATE, our method recovers the DR-Learner \citep{kennedy2020optimal}.
We show that this procedure is robust in the sense that the blackbox regression mimics having used the pseudo-outcome with the \emph{true} CDTE as the ``initial guess'', thus removing any bias or noise from fitting the baselines. We further show that our method is model-agnostic in the sense that, if the blackbox is not a universal approximator, we still get the best approximation to the CDTE function offered by the blackbox.
Lastly, we show that our estimating procedure allows for valid statistical inference on the best linear projection of CDTE, thus enabling interpretable analyses of distributional effect heterogeneity. We demonstrate in a comprehensive simulation study that we obtain uniformly better performance than the plug-in approach for several types of CDTEs. Finally, we apply our method on a real world study of 401(k) eligibility and its impact on financial wealth. 

\section{PROBLEM SETUP}

We consider either an experimental or observational dataset with two treatments, denoted by $0$ and $1$. Each unit in the dataset is a draw from a population of baseline covariates $X\in\Xcal$, treatment indicator $A\in\{0,1\}$, and observed outcome $\Yobs\in\RR$. The dataset consists of $n$ such independent draws, $Z_i=(X_i,A_i,Y_i)\sim Z=(X,A,Y),\,i=1,\dots,n$. We define the propensity score as $e^*(X)=\Prb{A=1\mid X}$. We assume throughout that $e^*(X)\in(0,1)$ almost surely, known as overlap.

Each unit is additionally associated with two unobserved potential outcomes, $\Ypot(0),\,\Ypot(1)\in\Rl$, representing the potential outcome we would observe if (possibly counter to fact) each treatment were applied. We assume we observe the potential outcome corresponding to the treatment indicator, $\Yobs=\Ypot(A)$, which also encapsulates an assumption of no interference between unit treatments. We assume unconfoundedness (ignorability) throughout: $\Ypot(a)\indep A\mid X$. For experimental data this is ensured by design via random assignment of $A$ (often with covariate-agnostic assignment, $A\indep X$). For observational data, this is an assumption that all potential sources of confounding are captured in $X$. For our purposes, the only difference between the two cases is whether the propensity score $e^*(X)$ is known.

We are interested in the differences between the conditional distributions of $\Ypot(1)$ and $\Ypot(0)$, given $X$. In the next section, we describe specific metrics for these differences.

\textbf{Notation.}
Given a distribution $F$, we define $\E_F[f(Z)]=\int f(z)dF(z)$.
We let $\widehat\E_n$ denote the empirical expectation $\widehat\E_nf(Z)=\frac{1}{n}\sum_{i=1}^n f(Z_i)$. 
For a parameter $f$, we reserve $f^*$ to represent its true value and $\widehat{f}$ a value learned from the data. We let $\|f\|:=\EE_F[f(z)^2]^{1/2}$ be the $L_2$ norm of $f$. We use $D$ to denote directional derivatives: $D_hF(h)\lvert_{h=h'}=\frac{\partial}{\partial a}F(h'+a)\lvert_{a=0}$, whenever this exists. If $h$ is a vector of functions, $D_hF(h)=(D_{h_1}F(h), ..., D_{h_k}F(h))$ and $D_{h_i}F(h)\lvert_{h=h'}=\frac{\partial}{\partial a}F((h'_1, ..., h'_i+a,...,h'_k))\lvert_{a=0}$. We let $\text{conv}(\mathcal S)$ denote the convex hull of $\mathcal S$. For two numbers $a, b$, we take $a\lesssim b$ to mean $a\leq Cb$ for some universal constant $C$ and $a\asymp b$ to mean $cb\leq a\leq Cb$ for some constants $c$ and $C$. Finally, we let $\overline{1, n}$ denote the set of integers $\{1,...,n\}$.

\section{CONDITIONAL DISTRIBUTIONAL TREATMENT EFFECTS}\label{cdte-sec}

CDTEs are functions mapping $x$ to a difference in some statistic of the conditional distributions of $\Ypot(1)$ and $\Ypot(0)$, given $X=x$. Examples of such statistics are the mean, yielding the CATE, and the $\tau$-quantile, yielding the CQTE. 
In this paper, we handle a very wide range of CDTEs given by statistics defined by \emph{moment equations} \citep{chamberlain1992efficiency, ai2003efficient}.

\begin{definition}[Moment Statistics]\label{def:statistic}
Given $\rho:\RR^{m+2}\to\RR^{m+1}$, we define a statistic of a distribution $F$ on $\RR$ as $\kappa^*(F)$ for $(\kappa^*(F),\,h^*(F))\in\RR\times\RR^m$ any solution (if it exists) to the moment equation
\begin{equation}\label{eq:kappamoment}
\E_F[\rho(Y,\kappa,h)]=0.
\end{equation}
\end{definition}

\begin{definition}[CDTEs]
Let $F_{Y(1)\mid X}$ and $F_{Y(0)\mid X}$ denote the conditional distributions of $Y(1)$ and $Y(0)$ given $X$, respectively.
Fix a statistic $\kappa^*(F)$ given by \cref{def:statistic}.
The corresponding CDTE is given by:
\begin{equation}
\kcdte(X)=
\kappa^*(F_{Y(1)\mid X})-\kappa^*(F_{Y(0)\mid X}).
\end{equation}
\end{definition}

For brevity, we will define the following functions (scalar-valued and $\R m$-valued, respectively)
\begin{align*}
\kappa^*_a(X)=\kappa^*(F_{Y(a)\mid X}), \hspace{1.5mm} h^*_a(X)=h^*(F_{Y(a)\mid X}), \hspace{1.5mm} a=0,1.
\end{align*}
We also assume these functions exist in that at least one solution to \cref{eq:kappamoment} exists for $F_{Y(a)\mid X}$ 
(see \cref{remark:multiplesolutions} regarding multiplicity).
In the examples below, we give specific names for $\kappa$ or $h$ (\eg, $q(F;\tau)$ for the $\tau$-quantile of $F$), in which case we use analogous abbreviations (\eg, $q_a(X;\tau)$).

Note that under unconfoundedness and overlap, $F_{Y(a)\mid X}$ is the same as $F_{Y\mid X,A=a}$, the conditional distribution of $Y$ given $X$ and $A=a$. Therefore, $\kappa^*_a$, $h_a^*$, and the CDTE are all identifiable from the data.  That is, despite being defined in terms of potential outcomes, the $\kcdte$ depends only on the distribution of observed data $(X,A,Y)$ and not on that of the unobserved data $(X, \Ypot(0), \Ypot(1))$, under our assumptions. The question we address in this paper is \textit{how} to learn $\kcdte$s from data.

We next review some important examples of CDTEs that fit into our framework. First note that CATE fits into this setup by setting $\rho(y,\kappa)=y-\kappa$ with $m=0$ in \cref{eq:kappamoment} (no $h$'s). The power of our framework is that it captures many CDTEs of interest beyond averages.

\subsection{Example 1: Conditional Quantile Treatment Effects}\label{sec:cqte}

Our first example is the \textbf{conditional quantile treatment effect} (CQTE). Given $\tau\in(0,1)$, the quantile at level $\tau$ of the distribution $F$ is defined as $q(F;\tau)=\inf\{y:F(y)\geq \tau\}$ (identifying $F$ with its cumulative distribution function).
Whenever $F$ has a positive derivative at $q(F;\tau)$, it is given by \cref{def:statistic} with $m=0$ (no $h$'s) and 
\begin{equation} \label{eq:cqte-moment}
    \rho(y,q)=\tau-\II[y\leq q].
\end{equation}

The corresponding CDTE is called the CQTE at level $\tau$.

Non-conditional QTEs are an important tool for quantifying the effects of treatments throughout the outcome distribution \citep{firpo, belloni, ldml}. This is especially important when we suspect that the outcome distribution might be skewed or heavy-tailed (\eg, income).

CQTEs offer an opportunity to assess such effects at the individual level. In particular, the CQTE can capture the potential increase in an individual's chance to have very poor outcomes due to treatment, even if the average effects are good or neutral. That is, if $A=1$ denotes an intervention, then CQTEs answer the prediction question: given an individual's covariates, how would intervening affect the 10\%-worst possible outcomes.

A related quantity is the difference between the $\tau$ and $1-\tau$ conditional quantiles across treatments: $\omega(X;\tau)=q(F_{Y(1)\mid X};\tau)-q(F_{Y(0)\mid X};1-\tau)$. This quantity can bound the (unobservable) individual treatment effect: 
$$\ts\Prb{\omega(X;\tau)\leq Y(1)-Y(0)\leq \omega(X;1-\tau)\mid X}\geq 1-4\tau.$$ 
For the sake of brevity and uniform treatment we focus on CDTEs (\ie, using the same statistic for both potential outcome distributions), but our results readily extend to quantities like this as well.

\subsection{Example 2: Conditional Super-Quantile Treatment Effects}\label{sec:csqte}

Our second example is the \textbf{conditional super-quantile treatment effect} (CSQTE). Given $\tau\in(0,1)$, the super-quantile (also known as conditional value-at-risk or tail expectation) at level $\tau$ of a distribution $F$ is defined as
\begin{align}\label{eq:csqte-opt}
\mu(F;\tau)&=\ts{\textstyle\inf_\beta}~\beta+\frac1{1-\tau}\int_{\beta}^\infty(1-F(y))\;dy\\
&={\textstyle\inf_\beta}~\EE_F[\beta+(1-\tau)^{-1}\max\{y-\beta,0\}].\tag*{}
\end{align}
The corresponding CDTE is called the CSQTE at level $\tau$.

Note that a $\beta$ realizing the above infimization is the quantile at level $\tau$, $q(F;\tau)$. Therefore, given positive density at the quantile, the super-quantile is given in \cref{def:statistic} by setting $m=1$ and
\begin{equation}\label{eq:csqte-moment}
\rho(y,\mu,q)=((1-\tau)^{-1}y\II[y\geq q]-\mu,\,\tau-\II[y\leq q]).
\end{equation} The super-quantile $\mu(F;\tau)$ is the largest-possible subpopulation average among all subpopulations comprising a $1-\tau$ fraction of the population of values described by $F$.
When $F(q(F;\tau))=\tau$, this can be phrased as the average of values above the $\tau$ quantile.
If we are interested in the left tail (average below a quantile), we can simply consider the the negative CSQTE in the negative outcomes.
Unlike the quantile, the super-quantile is a coherent risk measure \citep{artzner1999coherent}.

In particular, while the CQTE captures the impact on the outcomes at a single probability level, they can fail to provide a full picture of the risk profile as they are indifferent to anything beyond the threshold of the quantile. In contrast, the CSQTE captures effects beyond quantile breakpoint and quantifies the impact on the \emph{average}, say, 10\%-worst (or, best) outcomes. 
Crucially, since super-quantiles are a coherent risk measure, making individual decisions based on CSQTEs (\eg, intervene when the CSQTE is positive) is rational with respect to the coherent-risk axioms.

\subsection{Example 3: Conditional $f$-Risk Treatment Effects}\label{sec:cfrte}

Our third example is a whole class of coherent risk measures. Coherent risk measures quantify how bad/good a random loss/reward is while satisfying certain axioms (monotonicity, sub-additivity, homogeneity, and translational invariance). 
A key result is that coherent risk measures are equivalent to distributionally robust optimization \citep{ruszczynski2006optimization}. In this section, we focus on the \textbf{conditional $f$-risk treatment effect} (C$f$RTE), a family of coherent risk measures generated by $f$-divergences. 

Given a convex $f:\Rl\to\Rl$ with $f(1)=0$, we define the conditional $f$-risk at level $\delta\geq0$ as: 
\begin{align*}\ts
    &R^f(F;\delta) = \ts\sup_{G\ll F,\,D^f(G\|F)\leq \delta} \EE_{G}[Y],
\end{align*}
where $D^f(G\|F)=\EE_F[f(dG/dF)]$. For example, the $f$-divergence for $f(x)=x\log x$ is the Kullback–Leibler divergence, and the $f$-risk is known as entropic value-at-risk (EVaR) which upper bounds the super-quantile at level $1-e^{-\delta}$ \citep{ahmadi2012entropic}.
The $f$-risk admits a dual formulation given by the following convex optimization problem \citep{rockafellar1974conjugate}:
\begin{align*}
&\ts R^f(F;\delta)  = \inf_{\beta\geq0, \lambda\in \RR}~ \EE_F\left[m(Y,\beta,\lambda;\delta)\right],\\
& m(y,\beta,\lambda;\delta)=\delta \beta + \lambda + \beta f^*\left(\beta^{-1}\prns{y-\lambda}\right),
\end{align*}
where $f^*(x^*)=\sup_{x\in\RR}~xx^*-f(x)$ is the convex conjugate of $f$.

Therefore, under appropriate regularity, $R^f(F;\delta)$ is given by \cref{def:statistic} with
\begin{align}
\rho(y,R,\beta,\lambda)=&\ts\big(m(y,\beta,\lambda;\delta)-R^f,\,\frac{\partial }{\partial\beta}m(Z,\beta,\lambda;\delta),\tag*{}\\
&\;\ts\frac{\partial }{\partial\lambda}m(Z,\beta,\lambda;\delta)\big).
\end{align} 
The corresponding CDTE is called the C$f$RTE at level $\delta$. 
Frequently employed in actuarial science and finance, coherent risk measures are a fundamental tool for assessing risk. C$f$RTE can therefore be used to perform risk-benefit analyses on the treatment effect profiles of affected groups. 

\begin{figure}[t]
    \centering
    \includegraphics[width=\linewidth]{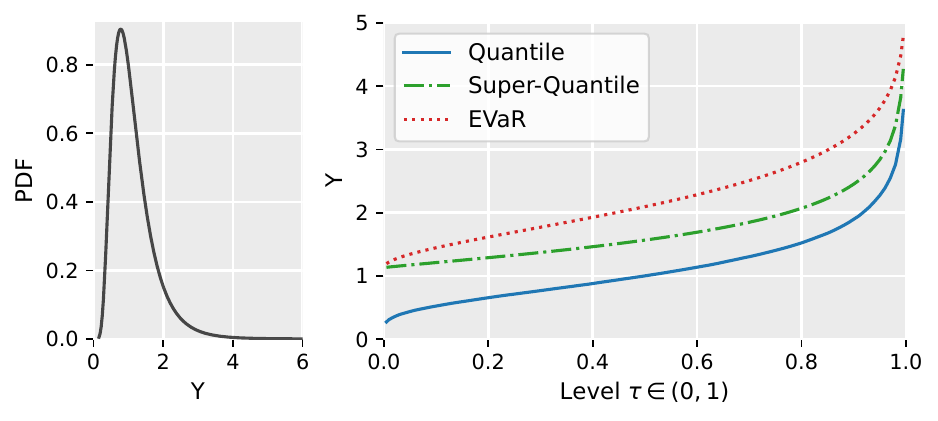}
    \vspace{-1.5em}
    \caption{Comparison of quantiles, super-quantiles, and EVaRs for a right-truncated ($Y\leq 6$) $\text{Lognormal}(\mu=0,$ $\;\sigma=0.5)$ at different risk levels $\tau\in(0, 1)$. \textit{Note:} Level $\tau$ corresponds to $\delta=-\log (1-\tau)$ for EVaR.}
    \label{fig:lognormal-risk}
     \vspace{-1em}
\end{figure}

\begin{remark}
    A natural question is which risk measure to use in practice. Ultimately, the appropriate risk measure (and level) will be application-dependent and it is up to the practitioner to select a measure that best reflects the desired risk profile. For illustration purposes, in \Cref{fig:lognormal-risk} we show how our three examples (quantiles, super-quantiles and EVaRs) compare for a heavy-tailed distribution.
\end{remark}

\section{RELATED LITERATURE}

\textbf{Learning CATEs.} CATE estimation is a central problem in causal learning and finds uses in both the analysis of causal interventions and decision support for personalization. Going beyond fully-parametric models, early advances relied on semiparametric models that imposed structure on the CATE function \citep{robins1992estimating, van2003unified, vansteelandt2014structural}. Recently there has been a surge of interest in leveraging machine learning for CATE estimation. These flexible methods either employ specific machine learning models such as Bayesian regression trees \citep{hill2011bayesian, hahn2020bayesian}, random forests (RFs) \citep{causalforest, oprescu2019orthogonal}, neural networks \citep{johansson2016learning, atan2018deep, shi2019adapting}, or allow for arbitrary blackbox meta-learners \citep{kunzel2019metalearners,nie2021quasi} by leveraging efficient influence functions \citep{robins2017minimax, kennedy2020optimal, curth2020estimating} and Neyman orthogonality \citep{chernozhukov2018double, foster2019orthogonal}. This paper is closest to the works on efficient influence functions and blackbox meta-learners, and we add to this literature by considering distributional statistics beyond averages. 

\textbf{Double machine learning and orthogonal statistical learning.} Another vein of related literature is learning with nuisances. Both our work and the above methods based on regressing efficient influence functions may be phrased within the wider framework of orthogonal statistical learning \citep{foster2019orthogonal}, which extends double machine learning (DML) \citep{chernozhukov2018double} from estimation to minimizing loss functions involving unknown nuisances subject to Neyman orthogonality \citep{neyman1959optimal}.
Neyman orthogonal losses arise naturally from efficient influence functions \citep{ichimura2022influence}, a fact that has been leveraged by \citet{foster2019orthogonal, kennedy2020optimal, curth2020estimating,athey2021policy} to tackle both CATE and policy learning. These methods have learning rates that adapt to the complexity of the target rather than that of the nuisances, but they largely focus on conditional averages. 
Our work builds on this line of research and is most similar to \citep{kennedy2020optimal} in that we propose nuisance agnostic CDTE estimators with guarantees beyond those given by Neyman orthogonality. 

\textbf{Distributional treatment effects.} The literature on DTEs can be split into two categories: (i) estimating cumulative distribution functions of potential outcomes and (ii) directly estimating distributional parameters of interest. In the first category, the main approach is to model conditional counterfactual distributions using distribution regression \citep{chernozhukov2013inference, chernozhukov2020network} or fully flexible approaches such as neural networks \citep{ge2020conditional, zhou2021estimating} and mean kernel embeddings \citep{park2021conditional}. Since these methods rely on plug-in estimation, they can be slow and biased when concerned with a particular DTE. 
The second category focuses on estimating a particular DTE.
Existing orthogonal/efficient methods focus on unconditional DTEs \citep{firpo,belloni,ldml}. Existing methods that tackle CDTEs rely on plug-in estimation \citep{park2021conditional} or parametric methods \citep{hohberg2020treatment}.
Our work bridges the gap by proposing robust, agnostic, and flexible CDTE learning.

\section{PSEUDO-OUTCOME REGRESSION FOR CDTEs}

In this section, we propose an algorithm for learning $\kcdte$s from data. As a first step, consider the following naive, but straightforward estimation procedure: 
\begin{equation}\label{plugin-alg}
\kcdte^{\text{Plugin}}(X)=
\widehat{\kappa}_1(X)-\widehat{\kappa}_0(X).
\end{equation}
where $\widehat{\kappa}_a(\cdot)$ are estimates for $\kappa_a(\cdot)$ (see \cref{sec:nuisance-est} regarding how these may be constructed).
This estimator is known as a ``plug-in'' estimator or, in the context of CATE learning, a ``T-learner'' \citep{kunzel2019metalearners}. Unfortunately, this approach has several drawbacks. One concern is that the treatment effect signal can be easily masked by noise in the baseline predictors. In particular, while many variables may strongly predict baseline response, only a few strongly modulate effect. The effect function may often be simpler, sparser, and/or smoother than each baseline function. Therefore, the plug-in estimator can suffer from excessive bias inherent in fitting baseline estimators in high-dimensions or using flexible models. If, on the other hand, we seek to use a simple model such as a linear fit, we will find that differencing the best linear predictors of baseline outcomes does not yield the best linear predictor of effect. For these reasons, it is imperative to learn CDTEs in a direct, model-agnostic, and robust way. 

To address the short-comings of the plug-in estimator, we consider the plug-in prediction on each data point $\kcdte^{\text{Plugin}}(X_i)$, debias it using the observed action and outcome $A_i,Y_i$, and finally regress the debiased prediction on $X$ again.
Our first task is to propose an appropriate debiased pseudo-outcome for CDTE learning.
\begin{definition}[$\kcdte$ Pseudo-Outcome]\label{pseudo-def}
Fix a statistic in \cref{def:statistic}.
Let $\nu^*_a = (\kappa^*_a, h^*_a)$ and
\begin{align*}
&\alpha^*_a(X) = (J_a^*(X))_1^{-1},\\
&\text{where}~~J_a^*(X)= D_{\nu_a}\{\EE[\rho(Y, \nu_a)\mid X, A=a]\} \big\lvert_{\nu_a=\nu_a^*}
\end{align*}
provided that $(J_a^*(X))^{-1}$ exists. Here, $(J_a^*(X))^{-1}_1$ denotes the first row of the inverse Jacobian.

Given some $e,\alpha,\nu$ serving as stand-ins for $e^*,\alpha^*,\nu^*$, 
we define the $\kcdte$ pseudo-outcome by
\begin{align}\label{eq:pseudooutcome}
    \psi(Z, e, \alpha, \nu) = &\kappa_1(X) - \kappa_0(X) \\
    &\ts- \frac{A-e(X)}{e(X)(1-e(X))}
    \alpha_A(X)^T\rho(Y, \nu_A(X)). \tag*{}
\end{align}
\end{definition} 
We refer to $e^*,\alpha^*,\nu^*$ as \emph{nuisance functions}, as they are unknown functions needed to construct our pseudo-outcome.
One initial motivation for \cref{eq:pseudooutcome} is that, by iterated expectations, we have $\E[\psi(Z,e^*,\alpha^*,\nu^*)\mid X]=\kcdte(X)$. 
That is, if we had plugged in the true nuisances, then the regression of our pseudo-outcome on $X$ yields exactly what we want.
This is, however, not so surprising because if we plug in $\nu=\nu^*$, the last term in \cref{eq:pseudooutcome} just has zero conditional expectation, given $X$, so we are left with $\kappa^*_1(X) - \kappa^*_0(X)$.
The reason why \cref{eq:pseudooutcome} is special is that, as we will show, if we make small errors in the nuisances, the impact on the conditional expectation of our pseudo-outcome is even smaller, leading to robustness guarantees. This is in contrast to the plug-in approach, where errors in $\widehat\kappa_a(X)$ propagate directly to $\kcdte^{\text{Plugin}}$, which is just their difference.

The origin of \cref{eq:pseudooutcome} is that $\psi(Z,e^*,\alpha^*,\nu^*)$ is in fact the efficient influence function for the estimand $\EE[\kcdte(X)]$ (for a review of influence functions see \citet{kennedy2022semiparametric,ichimura2022influence}). In particular, if our statistic is simply the mean ($\rho(y,\kappa)=y-\kappa$) then $\psi(Z,e^*,\alpha^*,\nu^*)$ reduces to the familiar doubly-robust influence function, $\E[Y\mid X,A=1]-\E[Y\mid X,A=0]+\frac{A-e^*(X)}{e^*(X)(1-e^*(X))}(Y-\E[Y\mid X,A])$ that produces the CATE when regressed on $X$ (as studied by \citet{kennedy2020optimal}).
As we never use the fact that $\psi(Z,e^*,\alpha^*,\nu^*)$ is the efficient influence function, we do not prove this or impose the necessary regularity conditions for this to actually hold precisely. Instead, we simply use a perturbation argument to essentially guess the form of \cref{eq:pseudooutcome}, given which we directly prove our robustness and inference guarantees.

\subsection{The CDTE Learning Algorithm}

\begin{algorithm}[t]\caption{CDTE Learner}\label{cdte-alg}
\begin{algorithmic}[1]
\Require Data $\{(X_i, A_i, Y_i): i\in \overline{1, n}\}$, folds $K\geq 2$, nuisance estimators, regression learner
 \For{$k\in \overline{1, K}$} 
    \State Use data $\{(X_i, A_i, Y_i): i\neq k-1\;(\text{mod}\; K)\}$ to 
    \State construct nuisance estimates $\hat e^{(k)},\hat \alpha^{(k)},\hat \nu^{(k)}$
    \For{$i= k-1\;(\text{mod}\; K)$}
        \State Set $\widehat{\psi}_i=\psi(Z_i,\hat e^{(k)},\hat \alpha^{(k)},\hat \nu^{(k)})$ 
    \EndFor
 \EndFor
 \State\Return~$\widehat{\kcdte}(x)=\widehat{\EE}_n[\widehat{\psi}\mid X=x]$
\end{algorithmic}
\end{algorithm}

We now describe our algorithm, which is summarized in \cref{cdte-alg}. We first split the data into $K$ even folds. We then construct a pseudo-outcome for each data point by plugging in estimates of the nuisances into \cref{eq:pseudooutcome}. The nuisances at a point are fit on data excluding the fold that the data point belongs to.
This ensures that the data point and the nuisance estimates are independent without splitting the data into two and instead only using parts of it for each task.
Finally, we regress the pseudo-outcome on $X$ using a given regressor. 
We use $\widehat{\EE}_n[W\mid X=x]$ to denote the function learned by the given regression method when regressing $W$ on $X$ given $n$ data points $(X_i,W_i)$, $i\in\overline{1, n}$.
This notation affords us significant generality. For example, the regression method may be to minimize the sum of squared errors over some function class (\eg, linear or neural nets) or it may be given by local polynomial regression or RFs.

\subsection{Nuisance Estimation}\label{sec:nuisance-est}

\Cref{cdte-alg} requires nuisance estimators as inputs. Exactly how these nuisances are estimated may depend on the particular scenario. Let us first discuss the propensity $e^*(x)$. If it is known, as in experimental settings, we may simply set it as our estimate. Otherwise, we can estimate it using probabilistic classification (\eg, logistic regression or neural nets with softmax output).

Next, we discuss $\nu_a^*(x)$. Most generally, since it is defined by solving the conditional moment restriction $\EE[\rho(Y,\nu_a^*(X))\mid X]=0$, this nuisance may be learned by employing methods made for solving such models \citep{ai2003efficient,chen2009efficient, bennett2019deep, bennett2020variational,athey2019generalized,khosravi2019non, dikkala2020minimax}. However, in some specific examples, more direct methods may be applicable. For both CQTE and CSQTE, $\nu_a^*(x)$ includes a conditional quantile function, which can be learned using any quantile regression method, whether minimizing the check loss \citep{koenker1978regression} or using forests \citep{meinshausen2006quantile}. For CSQTE, we additionally need to fit the conditional super-quantile. Per \cref{eq:csqte-moment}, that nuisance is given by the regression $\E[(1-\tau)^{-1}Y\II[Y\geq q_a(X;\tau)]\mid X,A=a]$. Therefore, one possibility is to split the training data (being all data excluding the $k\thh$ fold) into two halves, fit a conditional quantile estimate $\hat q^{(k)}_a(x;\tau)$ on one, set $\omega_i=(1-\tau)^{-1}Y_i\II[Y_i\geq \hat q^{(k)}_a(X_i;\tau)]$ on the other, and return $\hat\mu^{(k)}_a(x;\tau)=\widehat\EE_{(1-1/k)n/2}[\omega\mid X=x,A=a]$. In particular, a given $X$-regression method can simply be applied once to $A=0$ and once to $A=1$. Appendix A of \citet{dorn2021doubly} provides guarantees for this procedure when the regression learner minimizes squared error over a class with bracketing entropy and \citet{olma2021nonparametric} when using local linear regression.

Lastly, we discuss $\alpha_a^*(x)$. In some cases, it is a known function that need not be estimated. 
For C$f$RTE, it is equal to $(-1,0, 0)$. 
In other cases, it is given directly by other nuisances. For CSQTE, it is equal to $(-1,(1-\tau)^{-1}q^*_a(x;\tau))$ so we can simply re-use the estimate we constructed for $\nu_a^*$.
In yet other cases, it is another nuisance that must be estimated.
For CQTE, it is equal to $1/f_{Y\mid X=x,A=a}(q^*_a(x;\tau))$, the reciprocal of the density of $Y\mid X=x,A=a$ at the conditional quantile. One way to fit this suggested by \citet{leqi2021median} is to split the training data into two halves, fit a conditional quantile estimate $\hat q^{(k)}_a(x;\tau)$ on one, set $\omega_i=K((Y_i-q^{(k)}_a(X_i;\tau))/b_n)/b_n$ on the other, where $K(u)$ is a kernel function such as the standard normal density, and return $\hat\alpha^{(k)}_a(x)=\widehat\EE_{(1-1/k)n/2}[\omega\mid X=x,A=a]$.

\section{GUARANTEES FOR LEARNING}\label{sec:learning}

In this section, we study the finite sample error rates for Algorithm \ref{cdte-alg} with arbitrary first- and second-stage estimators. For this and the next section, let us fix some $\kcdte$ with pseudo-outcome as in Definition \ref{pseudo-def} and estimation procedures input to Algorithm \ref{cdte-alg}.
We require the following boundedness conditions.

\begin{assumption}[Boundedness]\label{cdte-assum}
For a nuisance realization set $\Xi$,
there exist $c_1>0,c_2\geq0,c_3\geq0,c_4>0,c_5\geq0$ and matrices $G, H\in\{0,1\}^{(m+1)\times(m+1)}$ such that $\forall (e,\alpha,{\nu}_a)\in\Xi,\;i,j,l\in\overline{1, m+1}, \;\overline{\nu}_a\in \text{conv}\{(\nu_a^*, {\nu}_a)\}$,
\begin{itemize}[itemsep=0mm,leftmargin=*,topsep=0mm]
\item $e^*(X), e(X)\in[c_1,\,1-c_1]$
\item $\left\lvert D_{\nu_{a,j}}\EE[\rho_i (Y, \nu_a)\mid X, A=a\lvert_{\nu_a=\overline{\nu}_a} \right\rvert\leq c_2G_{ij}$
\item $\left\lvert D_{\nu_{a,l}}D_{\nu_{a,j}}\EE[\rho_i (Y, \nu_a)| X=x, A=a]\lvert_{\nu_a=\overline{\nu}_a}\right\rvert \leq c_3H_{jl}$
\item $\det\fprns{D_{\nu_a}\{\EE[\rho(Y, \nu_a)\mid X=x, A=a]\}\big\lvert_{\nu_a=\overline{\nu}_a}}> c_4$ 
\item $|\rho_i(Y, \bar\nu_a)|\leq c_5$
\end{itemize}
\end{assumption}

The first condition in \cref{cdte-assum} ensures that both treatments and controls can be observed for any $X$ with some fixed probability. This is guaranteed in a randomized trial if $e^*(X)$ is a constant and otherwise is a standard assumption in observational studies. The other conditions encode the cross-term structure of the derivatives of our moments.
In most examples, the requirement of $\rho$ being bounded amounts to requiring $Y$ to be bounded. Similar boundedness assumptions are often made in debiased machine learning for ATE and CATE to control remainder terms.

We can then prove the following \emph{conditional} Neyman orthogonality for our pseudo-outcomes, which is the key step for learning guarantees.
\begin{theorem}[Conditional Neyman Orthogonality]\label{thm:condneym}
Suppose \cref{cdte-assum} holds and let $(e,\alpha,{\nu}_a)\in\Xi$. Then, 
\begin{align}
\notag&\left\|\EE\left[\psi(Z, e , \alpha, \nu)-\psi(Z, e^*, \alpha^*, \nu^*)\mid X\right]\right\|\lesssim \mathcal{E}(e , \alpha, \nu),\\
    & \ts
    \mathcal{E}(e , \alpha, \nu)=\label{eq:E}
    \sum_{a=0}^1 \big( \|\kappa_a-\kappa^*_a\| \; \|e -e^*\|\\
    & \ts\quad\phantom{\lesssim}+\sum_{i=1}^{m+1}\sum_{j=1}^{m+1}G_{ij}\|\alpha_{a,i} - \alpha^*_{a, i}\| \;\|\nu_{a,j} - \nu^*_{a, j}\|
    \notag\\
    & \ts\quad\phantom{\lesssim} + \sum_{i=1}^{m+1}\sum_{j=1}^{m+1}H_{ij}\|\nu_{a,i} - \nu^*_{a, i}\| \;\|\nu_{a,j} - \nu^*_{a, j}\|
    \big) ,\notag
\end{align}
where the $\lesssim$ hides dependence on $c_1,c_2,c_3,c_4$.
\end{theorem}

The result shows that whether we use the pseudo-outcome with oracle nuisances or estimated nuisances as a regression target, the difference is bounded by a \emph{quadratic} form in the nuisance errors, wherein $G,H$ govern which pairwise error products appear. Thus, even if nuisances are estimated slowly, the impact is marginal, as the error is squared.

We will show that, up to $\mathcal{E}=\sum_{k=1}^K\mathcal{E}(\hat e^{(k)},\hat \alpha^{(k)},\hat \nu^{(k)})$, our estimate $\widehat{\kcdte}$ produced by \cref{cdte-alg} behaves the same as if we applied our regression method to the pseudo-outcome with oracle nuisances, $\widetilde{\kcdte}(x)=\widehat{E}_n[\psi(Z,e^*,\alpha^*, \nu^*)\mid X=x]$. In particular, if $\mathcal E$ is generally smaller than the error $\|\widetilde{\kcdte}-\kcdte\|$, then the leading behavior of $\widehat{\kcdte}$ will be equivalent to that of $\widetilde{\kcdte}$. The particular statement of this will depend on what last-stage regression method we are using.

The significance is that \cref{cdte-alg} will behave like regressing an unbiased observation of $\kcdte$ on $X$, even though we used estimated nuisances.
First, this means we can learn $\kcdte$ at rates that match the complexity of that function. 
Second, this implies that we can directly approximate $\kcdte$ and get model-agnostic best-in-class guarantees. For example, if we use linear regression, we would get the \emph{best} linear approximation for $\kcdte$ at a rate of $n^{-1/2}$. This is especially important if we seek an interpretable model. In the next section, we show this also enables \emph{inference}.

\begin{remark}
As shown in Appendix \ref{sec:thm-app},
in many examples, \cref{eq:E} will contain only a few of the product terms, because many $G, H$ entries are $0$ and/or $\widehat{\alpha}(x, \widehat{\nu})=\alpha^*(x, \widehat{\nu})$. This enables us to trade off slower rates in one nuisance for faster rates in another while maintaining the same rate for $\mathcal E$.
\end{remark}
\begin{remark}\label{remark:multiplesolutions}
We never explicitly assume that the conditional moment restrictions identify $\nu^*_a(x)$ uniquely, only that they exist.
While uniqueness is not necessary, the right-hand side of \cref{eq:E} will not vanish unless $\widehat\nu\s k_a(x)$ converges to a \emph{single}, \emph{non-random} limit point $\nu^*_a(x)$. This is certainly possible even under solution multiplicity \citep{imbens2021controlling,kallus2022debiased}. At the same time, usually $\nu^*_a(x)$ will be unique under very mild assumptions such as having continuous distributions.
\end{remark}

First, we consider an empirical risk minimization algorithm (nonparametric least squares): given a class $\Fcal\subset[\Xcal\to\Rl]$,
\begin{align}\label{eq:erm}\ts
\widehat{\EE}_n[W\mid X=\cdot]\in\argmin_{f\in\mathcal F}\widehat\E_n(W-f(X))^2.
\end{align}

\begin{theorem}\label{thm:erm}
Suppose \cref{cdte-assum} holds and almost surely $(\hat e^{(k)},\hat \alpha^{(k)},\hat \nu^{(k)})\in\Xi$ for $k\in\overline{1, K}$.
Let $\widehat{\EE}_n[\cdot\mid X=x]$ be as in \cref{eq:erm} and suppose $\Fcal$ is convex, closed and has bracketing entropy $\log N_{[]}(\Fcal,\epsilon)\lesssim \epsilon^{-r}$ with $0<r<2$ and that $\abs{f(x)}\leq c_5\,\forall f\in\Fcal,x\in\Xcal$. Then,
\begin{align}
\ts\|\widehat{\kcdte}-\kcdte\|  \lesssim O_p(n^{-1/(2+r)})+\mathcal E.
\end{align}
\end{theorem}

The rate $O_p(n^{-1/(2+r)})$ is generally the rate for regressing a \emph{known} target using nonparametric least squares over $\Fcal$ with such bracketing entropy. For example, if $\Fcal$ is H\"older functions of smoothness $\beta$ in $d$-dimensional inputs then it satisfies the entropy condition with $r=d/\beta$ \citep[cor. 2.7.2]{vdv_wellner} and $n^{-\beta/(2\beta+d)}$ is the optimal rate for regression in such a class \citep{stone1982optimal}. Thus, if $\mathcal E=O_p(n^{-1/(2+r)})$, our regression behaves as though we supplied it with the oracle nuisances.

Obtaining such a rate for $\mathcal E$ is generally lax because it is \emph{quadratic} in nuisance-estimation errors. 
For example, if nuisances are estimated at the much slower rate $O_p(n^{-1/(4+2r)})$, the condition is ensured. The special structure in \cref{thm:condneym} further permits some trade-off between the rates of different nuisances. In Appendix \ref{sec:thm-app}, we give the pseudo-outcome for each of the examples in Section \ref{cdte-rates} and instantiate \cref{thm:condneym} to exactly characterize the trade-offs.

Going beyond empirical risk minimizers, leveraging \cref{thm:condneym} and invoking a result of \cite{kennedy2020optimal}
\footnote{The result appears as Theorem 1 in v2 of the arxiv preprint.}
we can characterize the behavior when using a last-stage regression method satisfying certain stability properties.

\begin{theorem}\label{cdte-rates} 
Suppose \cref{cdte-assum} holds and almost surely for $(\hat e^{(k)},\hat \alpha^{(k)},\hat \nu^{(k)})\in\Xi$ for $k\in\overline{1, K}$, and that, for any targets $Y$ and $W$,
        $\widehat{\EE}_n[Y\mid X=x]+c = \widehat{\EE}_n[Y+c\mid X=x]$
        and 
        $\|\widehat{\EE}_n[W\mid X] - \EE[W\mid X]\| \asymp \|\widehat{\EE}_n[Y\mid X] - \EE[Y\mid X]\|$ whenever $\EE[Y\mid X=x] = \EE[W\mid X=x]$.
    Then,
\begin{align}\label{eq:cdte-rates}
    &\ts\|\widehat{\kcdte}-\kcdte\|  \lesssim \|\widetilde{\kcdte}-\kcdte\| + \mathcal E.
\end{align}
\end{theorem}

\section{GUARANTEES FOR INFERENCE}\label{sec:inference}

We next study how \cref{cdte-alg} can be used for inference on the best linear projection of $\kcdte$:
$$\ts
\gamma^*\in\argmin_{\gamma\in\RR^p}\|\kcdte-\gamma\tr\phi(\cdot)\|,
$$
for a given $\phi:\Xcal\to\RR^p$. In particular, $\phi$ may be subset just some of the features. Linear projections offer an interpretable view into distributional-effect heterogeneity. In the previous section, we showed \cref{cdte-alg} can perform well at learning linear projections. Next, we show that we can further conduct inference on the coefficients, which can facilitate interpretation and credible conclusions.

\begin{theorem}[Asymptotic Normality of Linear Projections]\label{thm:proj-normal}
Suppose \cref{cdte-assum} holds and almost surely $(\hat e^{(k)},\hat \alpha^{(k)},\hat \nu^{(k)})\in\Xi$  for $k\in\overline{1, K}$.
Let $\widehat{\gamma}$ be the coefficient vector returned by \cref{cdte-alg} when using ordinary least squares (OLS) on $\phi(X)$ as the regression blackbox for the final stage.
Furthermore, assume
    $\|\widehat{e}\s k-e^*\|=o_p(1)$,
    $\|\widehat{\kappa}_a\s k-\kappa^*_a\|\|\widehat{e}\s k-e^*\|=o_p(n^{-1/2})$, $\|\widehat{\alpha}\s k_{a,i}-\alpha^*_{a, i}\|=o_p(n^{-1/4})$, $\|\widehat{\nu}\s k_{a,i}-\nu^*_{a,i}\|=o_p(n^{-1/4})$, $\forall i \in\overline{1,m+1},\, k\in\overline{1,K}$.
Then, $\widehat{\gamma}$ satisfies
\begin{equation}
    \sqrt{n}(\widehat{\gamma} - \gamma^*)\rightsquigarrow \Ncal (0, \Sigma^*),
\end{equation}
where $\Sigma^*$ is the asymptotic covariance matrix for the linear regression of $\psi(Z,e^*,\alpha^*, \nu^*)$ on $\phi(X)$.
\end{theorem}
\Cref{thm:proj-normal} implies that we can just use out-of-the-box OLS with the built-in OLS inference as the final stage regression in \cref{cdte-alg}. The inference results would remain valid, provided nuisances are estimated slowly but not too slowly. In particular, we should generally use robust (aka sandwich or Huber–White) standard errors, as we do not expect the projection to have homoskedastic errors.

\section{EMPIRICAL RESULTS}\label{sec:empirical}

In this section, we demonstrate our method and theoretical results by applying \cref{cdte-alg} to learn CSQTEs (\cref{sec:csqte}). We first benchmark its performance in simulated data and then illustrate its use in a study of 401(k) eligibility and its effect on wealth accumulation. We provide additional results for CQTEs and C$f$RTEs in \cref{sec:exp-results}. Replication code is available at \url{https://github.com/CausalML/CDTE}. 

\subsection{Simulation Study}

\begin{figure*}[t]
\vspace{-1em}
     \centering
     \begin{subfigure}[b]{0.29\textwidth}
         \centering
         \includegraphics[width=\textwidth]{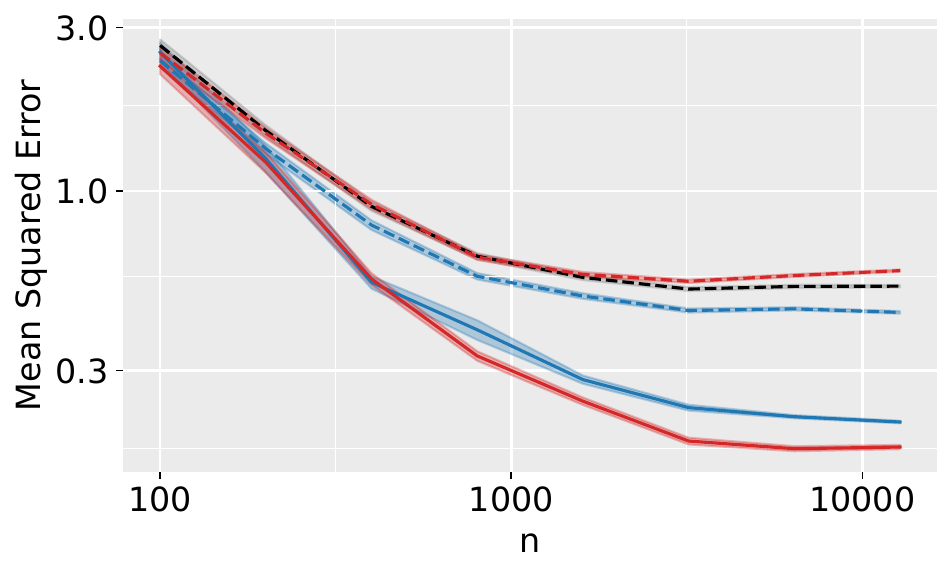}
     \end{subfigure}
     \begin{subfigure}[b]{0.28\textwidth}
         \centering
         \includegraphics[width=\textwidth]{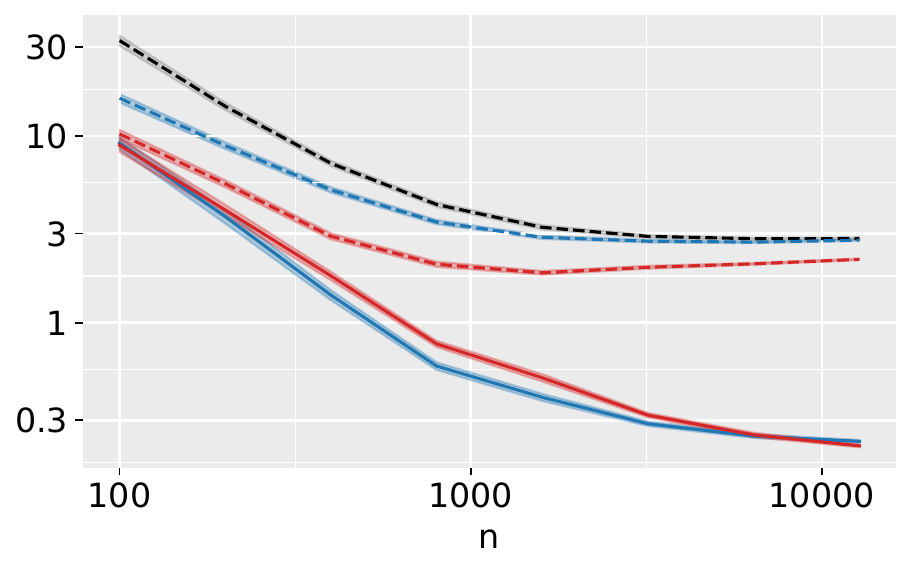}
     \end{subfigure}
     \begin{subfigure}[b]{0.39\textwidth}
         \centering
         \includegraphics[width=\textwidth]{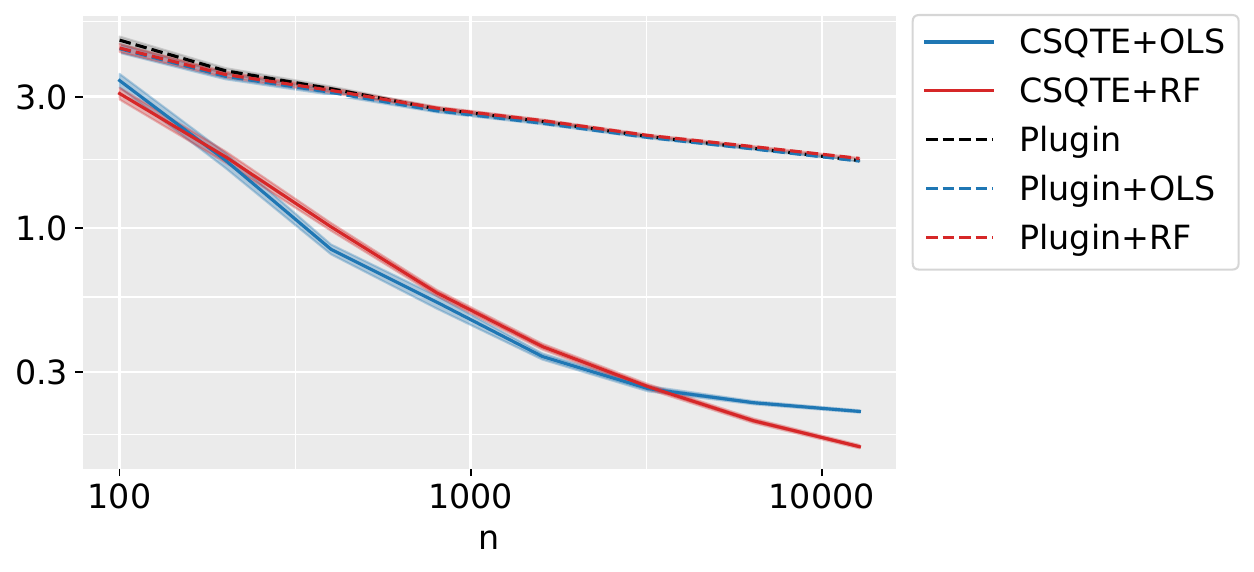}
     \end{subfigure}\\
     \begin{subfigure}[b]{0.29\textwidth}
         \centering
         \includegraphics[width=\textwidth]{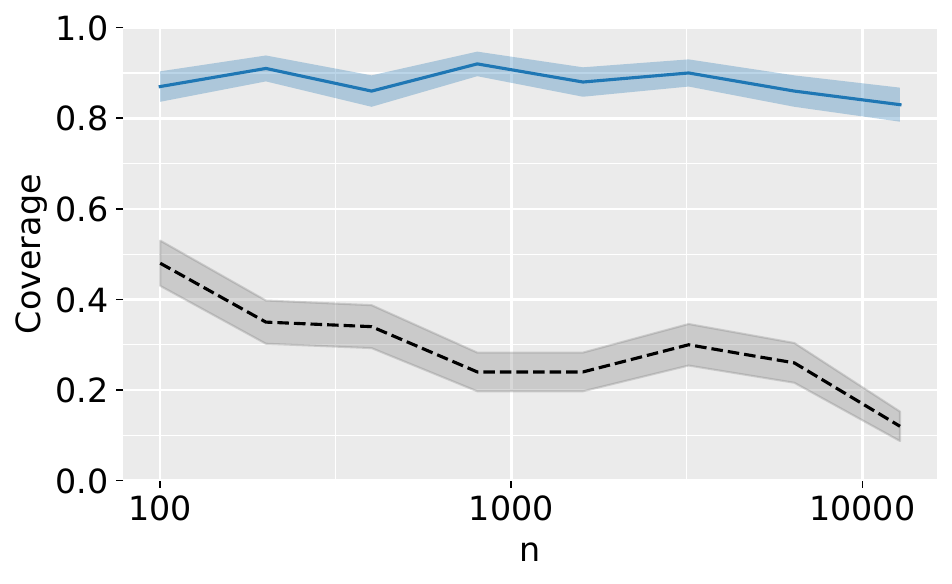}
         \caption{Flexible $\hat{\mu}$}
     \end{subfigure}
     \begin{subfigure}[b]{0.28\textwidth}
         \centering
         \includegraphics[width=\textwidth]{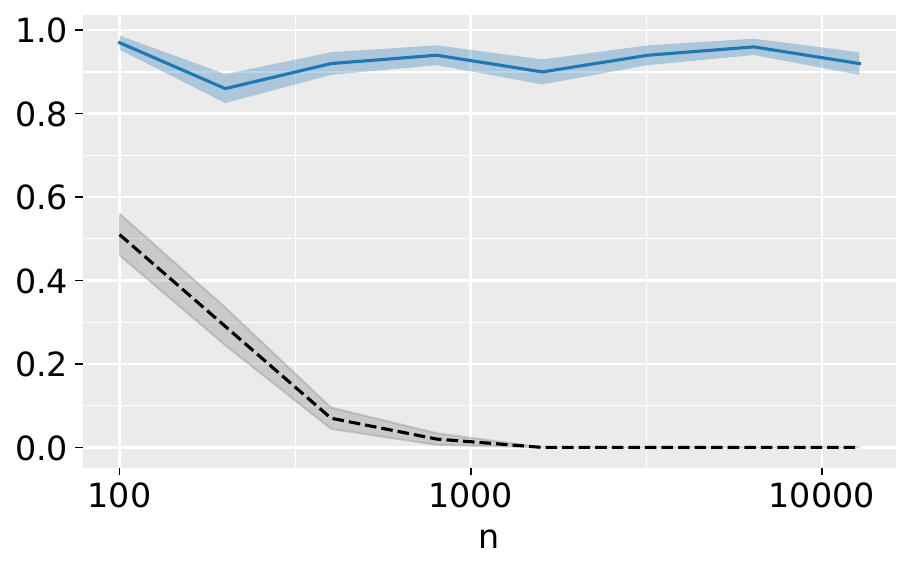}
         \caption{Misspecified $\hat{\mu}$}
     \end{subfigure}
     \begin{subfigure}[b]{0.39\textwidth}
         \centering
         \includegraphics[width=\textwidth]{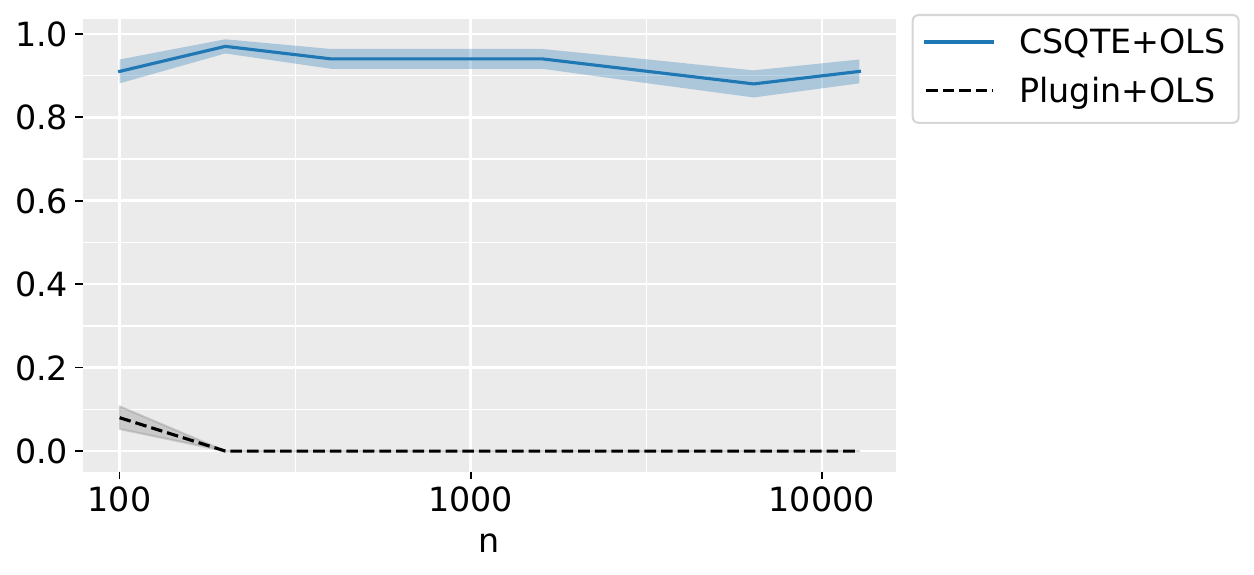}
         \caption{Slow $\hat{\mu}$}
     \end{subfigure}
        \caption{Mean squared error (MSE) and 95\% confidence interval coverage for different CSQTE learners. Shaded regions depict plus/minus one standard error over $100$ simulations.}
        \label{fig:csqte-sims}
\end{figure*}

We sample from the following data generating process:
\begin{align*}
    &X \sim \text{Unif}([0, 1]^{10}),\quad
    A  \sim \text{Bernoulli}(\sigma(6X_0-3)),\\
    &Y\mid X, A \sim \text{Lognormal}(X_0+AX_1, 0.2),
\end{align*}
where $\sigma$ is the logistic sigmoid. The coefficients in the expression for $A$ are chosen such that the true propensity lies in the range $[0.05, 0.95]$. We seek to measure CSQTE at level $\tau=0.75$, \ie, the conditional average of values above the $0.75$ conditional quantile. For this DGP, the true CSQTE at $\tau=0.75$ is given by $\mu_1(X; \tau)-\mu_0(X; \tau)= 1.29(e^{X_0+X_1}-e^{X_0})$, which is heterogeneous in $X$. 

We estimate ${e}(X)$ using logistic regression and ${q}_a(X;\tau)$ using a quantile random forest (QRF) \citep{meinshausen2006quantile}. We consider three options for estimating $\mu_a(X;\tau)$. First, we consider a \textit{flexible} learner we term the superquantile RF (SQRF). SQRF uses a RF to calculate weights (like QRF) and then computes \cref{eq:csqte-opt}, replacing the expectation by the weighted average over the data.
Second, we consider doing the same using a Gaussian kernel for calculating weights, choosing the bandwidth by Silverman's rule \citep{silverman2018density}. We refer to this as the \textit{slow} learner because it suffers badly from the curse of dimension.
Third, we consider estimating $\mu_a(X;\tau)$ using ordinary least squares for the regression $\frac{1}{1-\tau}\EE[Y\II[Y\geq \widehat{q}_a(x;\tau)]\mid X=x]$ using the estimated quantile $\widehat{q}_a$. We term this estimator a \textit{misspecified} learner since it is unlikely to fully capture the complexity of the superquantile (which is an exponential function of the features). For the final stage of \cref{cdte-alg}, we use either an ordinary least squares model (CSQTE+OLS) or a RF (CSQTE+RF). We set $K=5$ as the number of folds required by \cref{cdte-alg}. All forests use \texttt{scikit-learn} \citep{scikit-learn} defaults except for the minimum leaf size which is set to $n/20$ to control overfitting.

We compare the out-of-sample mean squared error (MSE) of the CSQTE estimator with that of the naive plug-in estimator from Eq. \ref{plugin-alg}.  To account for possible smoothing by the last stage regressor, we also construct Plugin+OLS and Plugin+RF given by running an additional OLS/RF model on the cross-fitted plug-in predictions. And, when the second stage algorithm is OLS, we check whether the 95\%-confidence interval OLS returns for the $X_1$ coefficient contains the coefficient from the true projection. The results are shown in  \Cref{fig:csqte-sims}. We run $100$ simulations for each $n=100,200,\ldots,12800$ and evaluate MSE over a fixed set of 500 random $X$ values. Our CSQTE learner provides uniformly strong MSE performance, and the results show this is not just a consequence of the second-stage regression. For inference, we achieve good coverage whereas plug-in approaches yield little to no coverage. These findings confirm our theoretical results: the quadratic dependence on nuisance errors enables oracle rates when the nuisances are estimated slowly or are misspecified (\cref{thm:condneym}, \cref{thm:erm}), and we can obtain valid inference when projecting onto linear spaces (\cref{thm:proj-normal}).

\subsection{Impact of 401(k) Eligibility on Financial Wealth}

\begin{figure*}[t]
\vspace{-1em}
\hfill\begin{minipage}[t]{0.27\linewidth}
\vspace{0pt}
\centering
\includegraphics[width=\textwidth]{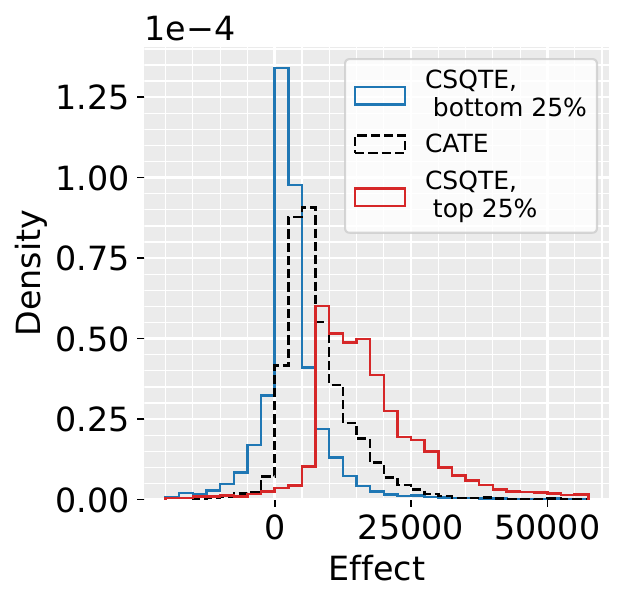}
\end{minipage}\hfill%
\hspace{1mm}
\begin{minipage}[t]{0.65\linewidth}
\vspace{10pt}
\centering
{\footnotesize\begin{tabular}[b]{ l c c c} \toprule
  \multirow{2}{*}{\textbf{Coefficient}} & \textbf{CSQTE} & \multirow{2}{*}{\textbf{CATE}} & \textbf{CSQTE} \\ 
  & Bottom 25\% & & Top 25\% \\ \midrule
 \multirow{2}{*}{\parbox{1cm}{Intercept\break{\scriptsize(\$10{,}000)}}} & $-0.021$ & $-0.95$ & $-2.07$\\
 & $(-1.06,\, 1.02)$ & $(-2.42,\, 0.51)$ & $(-7.04,\, 2.90)$ \\ 
 \multirow{2}{*}{Income} & $0.25^{**}$ & $0.21$ & $-0.05$ \\
  & $(0.08,\, 0.43)$ & $(-0.08,\, 0.50)$ & $(-1.12,\, 1.01)$\\ 
  \multirow{2}{*}{Age} &$105$  & $232^{**}$  & $513$ \\ 
   & $(-75,\, 286)$ & $(24,\, 441)$ & $(-182,\, 1210)$\\ 
   \multirow{2}{*}{Education} &$-801^{**}$ & $16$ & $1340$ \\ 
    &$(-1440,\, -164)$& $(-1050,\, 1090)$ & $(-2490,\, 5180)$\\ 
    \bottomrule
\end{tabular}}
\end{minipage}\hfill\vspace{-0.65em}
\captionlistentry[table]{entry for table}
\captionsetup{labelformat=andtable}
{\makeatletter\edef\@currentHref{table.caption.\the\c@table}\label{tab:csqte-401k-tab}}
\caption{Effect of 401(k) eligibility in dollars on the bottom 25\%, top 25\%, and average. The figure shows the distribution of CSQTEs and CATEs when using a random forest last stage. The table shows OLS inference on the projection of CSQTE and CATE on income, age, and education. "**" indicates statistical significance at level $0.05$ (\textit{p}-value < 0.05).}
\label{fig:csqte-401k-fig}
\vspace{-1em}
\end{figure*}

We apply the CSQTE estimator to study the impact of 401(k) eligibility on net financial assets. We use the dataset from \citet{chernozhukov2004effects}, which is based on the 1991 Survey of Income and Program Participation. The data contains 9{,}915 observations with 9 covariates such as age, income, education, family size, marital status, IRA participation, {\etc} While the eligibility for 401(k) (the treatment $A$) is not assigned at random, \citep{poterba1994401,chernozhukov2004effects} argue that unconfoundedness can be assumed conditional on the observed covariates. The outcome of interest ($Y$) is the net financial assets of an individual, defined as the sum of 401(k) balance, bank accounts and interest-earning assets minus non-mortgage debt. 

We apply the CSQTE estimator to understand effect heterogeneity beyond the conditiona  mean. Specifically, we estimate the average treatment effects on the bottom and top $25\%$ of financial asset holders, conditional on covariates. We analyze the results alongside CATE estimates given by a DR-Learner \citep{kennedy2020optimal} as implemented by \citet{econml}. For nuisance estimation, we use RF models with hyperparameters as in \citet{chernozhukov2018double}. The superquantile model is SQRF described above, the quantile model is QRF, the outcome learner for CATE is a RF regression and the propensity model is a RF classifier. 

We consider two options for second-stage regressions.
First, we consider using an RF model using all 9 covariates. We train the three estimators (CATE and upper/lower CSQTEs), and plot the distribution of predicted conditional effects on the 9{,}915 observations in \Cref{fig:csqte-401k-fig}.  
We observe that 401(k) eligibility has a much higher financial impact on the high end of the conditional net worth distributions as compared to those on the lower end. While the CATE distribution reassuringly lies in between the bottom 25\% and top 25\% distributions, it fails to capture this disparity in effects.  

To understand the heterogeneity in CSQTEs and CATEs, we use an OLS projection in the final stage. We choose the top three features that the RF last stage picked up as most important (see \Cref{fig:app-401k-rf-feature-importance} in Appendix \ref{sec:exp-results}): income, age and education. The resulting coefficients and 95\% confidence intervals are depicted in \Cref{tab:csqte-401k-tab}. 
For the bottom 25\%, income and education are the main (statistically significant) drivers of the average effects. The positive income coefficient suggests that 401(k) eligibility provides more gains to those who potentially have the funds to invest. Likewise, the negative education coefficient means that the effects are higher among less educated earners.
We hypothesize that higher educated earners might have a more comprehensive financial education and would save and invest regardless of 401(k) eligibility, whereas 401(k) availability provides lower educated earners a default investment option.
For the top 25\% asset holders, higher age and education rather than income have the largest impact on 401(k) eligibility returns.
This group also displays more variability as none of the coefficients are statistically significant. The CATEs provide middle-of-the-road estimates that miss out on trends at the two ends of the spectrum. Thus, CSQTEs are a powerful tool for uncovering different behaviours across the conditional distribution that a simple average would obscure. 

\section{CONCLUDING REMARKS}

We provide new tools to assess distributional effect heterogeneity through agnostic and robust learning of CDTEs in a generic framework that includes quantiles, super-quantiles, and $f$-risk measures. These tools can be used to analyze treatments as well as to support personalized decisions that take risk into account.
We provide strong guarantees for both learning and inference, and demonstrate our methods in both synthetic and real data. We further discuss some limitations of our work in Appendix \ref{sec:limitations}.

\subsubsection*{Acknowledgements}
This material is based upon work supported by the National Science Foundation under Grant No. 1939704 and the U.S. Department of Energy, Office of Science, Office of Advanced Scientific Computing Research, under Award Number DE-SC0023112.

\bibliography{ref}
\bibliographystyle{abbrvnat}

\appendix
\onecolumn

\section{PROOFS OF MAIN THEOREMS}\label{sec:thm-proofs}

\subsection{Proof of Theorem \ref{thm:condneym}}

We first note the following useful identities:

\begin{enumerate}
    \item (Functional analog of Taylor's expansion - 2\textsuperscript{nd} order). Let $F:\Fcal^d\rightarrow\RR$ where $\Fcal$ is a vector space of functions. For any $h, h'\in \Fcal^d$, assume $t\mapsto F(t\cdot h+(1-t)\cdot h')$ has second order derivatives in an open interval containing $[0,1]$, then $\exists \; \overline{h}\in \text{conv}(\{h, h'\})$ such that 
    \begin{equation*}
        F(h') = F(h) + \sum_{i=1}^d D_{h_i}F(h)(h'_i-h_i) + \sum_{i=1}^d \sum_{j=1}^d D_{h_i}D_{h_j}F(\overline{h})(h'_i-h_i)(h'_j-h_j)
    \end{equation*}
    \item For all $i\in \overline{1,m}, a\in\{0, 1\}, x\in \Xcal$, the following hold:
    \begin{align}
        &\EE[\rho_i(Y, \nu_a^*)\mid X=x, A=a] = 0 \tag{Moment equations}\\
        &\sum_{i=1}^{m+1} \alpha_{a, i}(x, \nu^*_a)D_{\nu_{a,j}}\EE[\rho_i(Y, \nu_a^*)\mid X=x, A=a] = \II[j=1] \tag{Definition of $\alpha$}
    \end{align}
\end{enumerate}

\begin{proof}
We wish to bound the term:
\begin{align*}
     \mathcal{E}(e , \alpha, \nu) = \left\|\EE[\psi(Z, e, \alpha, \nu)\mid X=x]-\EE[\psi(Z, e^*, \alpha^*, \nu^*)\mid X=x]\right\|
\end{align*}
To simplify our calculations, we write $\psi$ as a difference of two terms $\psi^1 - \psi^0$ given by:
\begin{align*}
    \psi^1(Z, e, \alpha, \nu) &= \kappa_1(X) - \frac{A}{e(X)}\sum_{i=1}^{m+1} \alpha_{1, i}(X, \nu_{1, i})\rho(Y, \nu_{1, i})\\
    \psi^0(Z, e, \alpha, \nu) &= \kappa_0(X) - \frac{1-A}{1-e(X)}\sum_{i=1}^{m+1} \alpha_{1, i}(X, \nu_{0, i})\rho(Y, \nu_{0, i})
\end{align*}
With this notation, we have:
\begin{align*}
     \mathcal{E}(e , \alpha, \nu)\lesssim
     \sum_{a=0}^1  \left\|\EE\left[\psi^a(Z, e, \alpha, \nu)-\psi^a(Z, e^*, \alpha^*, \nu^*)\mid X=x\right]\right\|
\end{align*}
The inequality above follows from the inequalities $(a+b)^2\leq 2(a^2+b^2)$ and $\sqrt{a+b}\leq \sqrt{a}+\sqrt{b}$ (for non-negative $a, b$). Without loss of generality, we consider the bound for $\psi^1$. The proof for $\psi^0$ is very similar. We have:
\begin{align*}
    & \EE\left[\psi^1(Z, e, \alpha, \nu)-\psi^1(Z, e^*, \alpha^*, \nu^*)\mid X=x\right]  \\
    & \qquad =\sum_{a=0}^1 \EE\left[\psi^1(Z, e, \alpha, \nu)-\psi^1(Z, e^*, \alpha^*, \nu^*)\mid X=x, A=a\right] P(A=a\mid X=x)\\
    & \qquad = \kappa_1(x)- \kappa_1^*(x) -\frac{e^*(x)}{e(x)}\sum_{i=1}^{m+1}
                \alpha_{1, i}(x, \nu_1)\EE[\rho_i(Y, \nu_1)\mid X=x, A=1]\\
    & \qquad = \kappa_1(x)- \kappa_1^*(x) -
                \underbrace{\frac{e^*(x)}{e(x)}\sum_{i=1}^{m+1}
                (\alpha_{1, i}(x, \nu_1) - \alpha^*_{1, i}(x, \nu^*_1))\EE[\rho_i(Y, \nu_1)\mid X=x, A=1]}_{\Lambda_1}\\
    & \hspace{9.5em}  -\underbrace{\frac{e^*(x)}{e(x)}\sum_{i=1}^{m+1}
                \alpha^*_{1, i}(x, \nu^*_1)\EE[\rho_i(Y, \nu_1)\mid X=x, A=1]}_{\Lambda_2}
\end{align*}
where in the last equality we subtracted and then added back a term. We now study $\Lambda_1$ and $\Lambda_2$. For both these quantities, we do a Taylor expansion of $\EE[\rho_i(Y, \nu_1)\mid X=x, A=1]$ around $\nu_1^*$. For $\Lambda_1$, it suffices to use the Taylor expansion to first order only. For $\Lambda_2$, we perform a second order expansion:
\begin{align*}
    &\EE[\rho_i(Y, \nu_1)\mid X=x, A=1] = \EE[\rho_i(Y, \nu^*_1)\mid X=x, A=1] \\
    & \hspace{5em} + \sum_{j=1}^{m+1} D_{\nu_j} \EE[\rho_i(Y, \overline{\nu})\mid X=x, A=1] (\nu_{1,j}(x) - \nu^*_{1,j}(x)) \tag{first order expansion}\\
    & \EE[\rho_i(Y, \nu_1)\mid X=x, A=1] = \EE[\rho_i(Y, \nu^*_1)\mid X=x, A=1] \\
    & \hspace{5em} + \sum_{j=1}^{m+1} D_{\nu_j} \EE[\rho_i(Y, \nu^*_1)\mid X=x, A=1] (\nu_{1,j}(x) - \nu^*_{1,j}(x)) \\
    & \hspace{5em} + \frac{1}{2}\sum_{j=1}^{m+1}\sum_{l=1}^{m+1} D_{\nu_j}D_{\nu_l} \EE[\rho_i(Y, \overline{\nu})\mid X=x, A=1](\nu_{1,j}(x) - \nu^*_{1,j}(x))(\nu_{1,l}(x) - \nu^*_{1,l}(x))
    \tag{second order expansion}
\end{align*}
Then $\Lambda_1$ is given by:
\allowdisplaybreaks
\begin{align*}
    \Lambda_1 & = \frac{e^*(x)}{e(x)}\sum_{i=1}^{m+1}
                (\alpha_{1, i}(x, \nu_1) - \alpha^*_{1, i}(x, \nu^*_1))\EE[\rho_i(Y, \nu_1)\mid X=x, A=1] \\
                & = \frac{e^*(x)}{e(x)}\sum_{i=1}^{m+1}
                (\alpha_{1, i}(x, \nu_1) - \alpha^*_{1, i}(x, \nu^*_1))\Bigg(\cancelto{0}{\EE[\rho_i(Y, \nu^*_1)\mid X=x, A=1]} \\
                & \hspace{13em}+ \sum_{j=1}^{m+1} D_{\nu_j} \EE[\rho_i(Y, \overline{\nu})\mid X=x, A=1] (\nu_{1,j}(x) - \nu^*_{1,j}(x))\Bigg) \tag{first order expansion}\\
                & = \frac{e^*(x)}{e(x)} \sum_{i=1}^{m+1}\sum_{j=1}^{m+1}D_{\nu_j} \EE[\rho_i(Y, \overline{\nu})\mid X=x, A=1]  (\alpha_{1, i}(x, \nu_1) - \alpha^*_{1, i}(x, \nu^*_1))(\nu_{1,j}(x) - \nu^*_{1,j}(x))\\
    \|\Lambda_1\| &\leq \frac{c_2}{c_1} \sum_{i=1}^{m+1}\sum_{j=1}^{m+1}G_{ij}\|\alpha_{1,i} - \alpha^*_{1,i}\| \;\| \nu_{1,j} - \nu^*_{1, j}\|\\
    & \lesssim \sum_{i=1}^{m+1}\sum_{j=1}^{m+1} G_{ij}\|\alpha_{1,i} - \alpha^*_{1,i}\| \;\| \nu_{1,j} - \nu^*_{1, j}\|
\end{align*}
We proceed in a similar way for $\Lambda_2$, but this time using the second order expansion:
\begin{align*}
    \Lambda_2 &= \frac{e^*(x)}{e(x)}\sum_{i=1}^{m+1}
                \alpha^*_{1, i}(x, \nu^*_1)\EE[\rho_i(Y, \nu_1)\mid X=x, A=1]\\
                & = \frac{e^*(x)}{e(x)}\sum_{i=1}^{m+1}
                \alpha^*_{1, i}(x, \nu^*_1)\Bigg(
                \cancelto{0}{\EE[\rho_i(Y, \nu^*_1)\mid X=x, A=1]} \\
                & \hspace{5em} + \sum_{j=1}^{m+1} D_{\nu_j} \EE[\rho_i(Y, \nu^*_1)\mid X=x, A=1] (\nu_{1,j}(x) - \nu^*_{1,j}(x)) \\
                & \hspace{5em} + \frac{1}{2}\sum_{j=1}^{m+1}\sum_{l=1}^{m+1} D_{\nu_j}D_{\nu_l} \EE[\rho_i(Y, \overline{\nu}_1)\mid X=x, A=1](\nu_{1,j}(x) - \nu^*_{1,j}(x))(\nu_{1,l}(x) - \nu^*_{1,l}(x))
                \Bigg)\\
                & = \frac{e^*(x)}{e(x)}\sum_{j=1}^{m+1}\underbrace{\sum_{i=1}^{m+1}
                \alpha^*_{1, i}(x, \nu^*_1) D_{\nu_j} \EE[\rho_i(Y, \nu^*_1)\mid X=x, A=1]}_{=\II[j=1]}(\nu_{1,j}(x) - \nu^*_{1,j}(x)) \tag{see useful identities}\\
                & \quad + \frac{1}{2}\frac{e^*(x)}{e(x)}\sum_{i=1}^{m+1}
                \sum_{j=1}^{m+1}\sum_{l=1}^{m+1} \Bigg(\alpha^*_{1, i}(x, \nu^*_1) D_{\nu_j}D_{\nu_l} \EE[\rho_i(Y, \overline{\nu}_1)\mid X=x, A=1]\\
                & \hspace{12em} \cdot (\nu_{1,j}(x) - \nu^*_{1,j}(x))(\nu_{1,l}(x) - \nu^*_{1,l}(x))
                \Bigg)\\
                & = \frac{e^*(x)}{e(x)}(\kappa_1(x) - \kappa_1^*(x))\\
                & \quad + \frac{1}{2}\frac{e^*(x)}{e(x)}\sum_{i=1}^{m+1}
                \sum_{j=1}^{m+1}\sum_{l=1}^{m+1} \Bigg(\alpha^*_{1, i}(x, \nu^*_1) D_{\nu_j}D_{\nu_l} \EE[\rho_i(Y, \overline{\nu}_1)\mid X=x, A=1]\\
                & \hspace{12em} \cdot (\nu_{1,j}(x) - \nu^*_{1,j}(x))(\nu_{1,l}(x) - \nu^*_{1,l}(x))
                \Bigg)
\end{align*}
Adding back the $\kappa_1$ terms to $\Lambda_2$, we have:
\begin{align*}
    \kappa_1(x) - \kappa_1(x) - \Lambda_2 & = \frac{1}{e(x)}(e(x)-e^*(x))(\kappa_1(x) - \kappa_1^*(x))\\
    & \quad + \frac{1}{2}\frac{e^*(x)}{e(x)}\sum_{i=1}^{m+1}
                \sum_{j=1}^{m+1}\sum_{l=1}^{m+1} \Bigg(
                \alpha^*_{1, i}(x, \nu^*_1) D_{\nu_j}D_{\nu_l} \EE[\rho_i(Y, \overline{\nu}_1)\mid X=x, A=1]\\
                & \hspace{12em} \cdot (\nu_{1,j}(x) - \nu^*_{1,j}(x))(\nu_{1,l}(x) - \nu^*_{1,l}(x))
                \Bigg)\\
    \|\kappa_1(x) - \kappa_1(x) - \Lambda_2\| & \leq \frac{1}{c_1}\|\kappa_1-\kappa_1^*\|\|e-e^*\| + \frac{c_3c_4}{2c_1}\sum_{j=1}^{m+1}\sum_{l=1}^{m+1} H_{jl} \|\nu_{1,j} - \nu^*_{1,j}\| \; \|\nu_{1,l} - \nu^*_{1,l}\|\\
    & \lesssim \|\kappa_1-\kappa_1^*\|\|e-e^*\| + \sum_{j=1}^{m+1}\sum_{l=1}^{m+1} H_{jl}\|\nu_{1,j} - \nu^*_{1,j}\| \; \|\nu_{1,l} - \nu^*_{1,l}\|
\end{align*}
Constants $c_1, c_2,c_3,c_4$ and binary matrices $H, G$ are defined in Assumption \ref{cdte-assum}. Putting everything together, we have:
\begin{align*}
    \left\|\EE\left[\psi^1(Z, e, \alpha, \nu)-\psi^1(Z, e^*, \alpha^*, \nu^*)\mid X=x\right] \right\| &= \|\kappa_1(x) - \kappa_1(x) - \Lambda_2 - \Lambda_1\|\\
    &  \lesssim \|\kappa_1(x) - \kappa_1(x) - \Lambda_2\| + \|\Lambda_1\|\\
    & \lesssim  \|\kappa_1-\kappa_1^*\|\|e-e^*\| + \sum_{j=1}^{m+1}\sum_{l=1}^{m+1}H_{jl} \|\nu_{1,j} - \nu^*_{1,j}\| \; \|\nu_{1,l} - \nu^*_{1,l}\| \\
    & \hspace{9.4em} + \sum_{i=1}^{m+1}\sum_{j=1}^{m+1}G_{ij}\|\alpha_{1,i} - \alpha^*_{1,i}\| \;\| \nu_{1,j} - \nu^*_{1, j}\|
\end{align*}
Putting the $\psi^0$ and $\psi^1$ bounds together, we obtain the desired result:
\begin{align*}
    \mathcal{E}(e,\alpha, \nu)
    & \lesssim \sum_{a=1}^1 \Bigg( \|\kappa_a-\kappa^*_a\| \; \|e -e^*\|
    +\sum_{i=1}^{m+1}\sum_{j=1}^{m+1}G_{ij}\|\alpha_{a,i} - \alpha^*_{a, i}\| \;\|\nu_{a,j} - \nu^*_{a, j}\|
    \tag*{}\\
    & \hspace{12.4em} + \sum_{i=1}^{m+1}\sum_{j=1}^{m+1}H_{ij}\|\nu_{a,i} - \nu^*_{a, i}\| \;\|\nu_{a,j} - \nu^*_{a, j}\|
    \Bigg) \tag*{}
\end{align*}
\end{proof}

\subsection{Proof of Theorem \ref{thm:erm}}
\begin{proof}
Let $I_k=\{i:i= k-1\;(\text{mod}\; K)\}$ and $I^C_k=\{i:i\neq k-1\;(\text{mod}\; K)\}$ be the $k^{\text{th}}$ data fold and its complement, let $n_k=\abs{I_k}\in\{\lfloor n/K\rfloor,\lceil n/K \rceil\}$, and let $\lambda_k=n_k/n$ (note none of $I_k,n_k,\lambda_k$ are random). Let $\PP g(Z)$ denote expectation over $Z$ alone (that is, conditioning on the data, should $g$ depend on it) and $\hat \PP_kg(Z)$ denote empirical expectation over $I_k$. Further, let $\PP(g(Z)\mid X)$ denote the conditional expectation, $\|g\|_2=\PP g^2$, and $\Pi_\Fcal(g)\in\argmin_{f\in\Fcal}\|f-g\|_2$.

Define the pseudo-outcome random variables: 
\begin{align*}
\psi&=\psi(Z, e^*, \alpha^*, \nu^*),\\
\hat\psi_k&=\psi(Z, \hat e^{(k)},\hat \alpha^{(k)},\hat \nu^{(k)}),\\
\hat\psi&=\sum_{k=1}^K\lambda_k\hat\psi_k.
\end{align*}
Further, define the squared-error objectives: 
\begin{align*}
\hat R_k(f)&=\hat \PP_k(f(X)-\hat\psi_k)^2,\\
\hat R(f)&=\sum_{k=1}^K\lambda_k\hat R_k(f),\\
\tilde R_k(f)&=\PP(f(X)-\hat\psi_k)^2,\\
\tilde R(f)&=\sum_{k=1}^K\lambda_k\tilde R_k(f),\\
\bar R(f)&=\PP (f(X)-\hat\psi)^2,\\
R(f)&=\PP(f(X)-\psi)^2.
\end{align*}
Note that $\tilde R(f)-\bar R(f)=\tilde R(f')-\bar R(f')$ for any $f,f'$, that is, $\tilde R(f)$ and $\bar R(f)$ are equal up to additive constants.
Finally, we have the predictors:
\begin{align*}
\hat f&\in\argmin_{f\in\Fcal}\hat R(f),\\
\bar f&=\Pi_{\Fcal}(\PP(\hat\psi\mid X))\in\argmin_{f\in\Fcal}\bar R(f)=\argmin_{f\in\Fcal}\tilde R(f),\\
f^*&=\Pi_{\Fcal}(\PP(\psi\mid X))\in\argmin_{f\in\Fcal} R(f).
\end{align*}
Note $\hat f=\widehat{\kcdte}$ and $f^*=\kcdte$; we rename them for brevity.

We then have:
\begin{align*}
\magd{\hat f-f^*}&\leq\magd{\hat f-\bar f}+\magd{\bar f-f^*}
\\&=\magd{\hat f-\bar f}+\magd{\Pi_{\Fcal}(\PP(\hat\psi\mid X))-\Pi_{\Fcal}(\PP(\psi\mid X))}\\
&\leq \underbrace{\magd{\hat f-\bar f}}_{\mathrm{Error}_A}+\underbrace{\magd{\PP(\hat\psi\mid X)-\PP(\psi\mid X)}}_{\mathrm{Error}_B},\end{align*}
where the inequality is because $\Fcal$ is closed convex

We first tackle $\mathrm{Error}_B$:
\begin{align*}
\magd{\PP(\hat\psi\mid X)-\PP(\psi\mid X)}
&=\magd{\sum_{k=1}^K\lambda_k(\PP(\hat\psi_k\mid X)-\PP(\psi\mid X))}
\\&\leq\sum_{k=1}^K\lambda_k\magd{\PP(\hat\psi_k\mid X)-\PP(\psi\mid X)},
\end{align*}
which we bound as $\lesssim \mathcal E$ by \cref{thm:condneym}.

Next, we address $\mathrm{Error}_A$. Note that $\magd{\hat f-\bar f}\geq t$ means that $\tilde R(\hat f)-\tilde R(\bar f)\geq t^2$ and $\hat R(\hat f)-\hat R(\bar f)\leq 0$. By intermediate value theorem and convexity of $\tilde R,\hat R,\Fcal$, we have for some $f\in[\hat f,\bar f]$ that $\tilde R(f)-\tilde R(\bar f)=t^2$ and $\hat R(f)-\hat R(\bar f)\leq 0$. This in turn implies that $\exists f\in\Fcal$ with $\magd{f-\bar f}\leq t$ and $(\tilde R(f)-\tilde R(\bar f))-(\hat R(f)-\hat R(\bar f))\geq t^2$. Because $\tilde R-\hat R$ is average of $\tilde R_k-\hat R_k$, this in turn implies that for at least one $k\in\overline{1, K}$ we have $\exists f\in\Fcal$ with $\magd{f-\bar f}\leq t$ and $(\tilde R_k(f)-\tilde R_k(\bar f))-(\hat R_k(f)-\hat R_k(\bar f))\geq t^2$. 
Since $\abs{\psi}\leq c_5,\abs{\hat\psi_k}\leq c_5,\abs{f(x)}\leq c_5$, we have $(f(X)-\hat\psi_k)^2-(\bar f(X)-\hat\psi_k)^2\leq 4c_5\abs{f(X)-\bar f(X)}$. Now, consider $g(Z)=\frac{(f(X)-\bar f(X))(f(X)+\bar f(X)-2\hat\psi_k)}{4c_5}$. Then $(\PP-\hat\PP_k)g\geq t^2/(4c_5)$ and $\|g\|\leq \|f-\bar f\|\leq t$ since $\|f(X)+\bar f(X)-2\hat\psi_k\|\leq 4c_5$.

Therefore, by union bound and iterated expectations,
\begin{align*}
& \Prb{\magd{\hat f-\bar f}\geq t}\\
&\leq \sum_{k=1}^K\EE\;\Prb{\exists g\in\left\{\frac{(f-\bar f)(f+\bar f-2\hat\psi_k)}{4c_5}:f\in\Fcal\right\}~:~\PP g^2\leq t^2,~(\PP-\hat\PP_k)g\geq t^2/(4c_5) \mid \{Z_i:i\in I_k^C\}}.
\end{align*}
By lemma 3.4.2 of \citet{vdv_wellner} and Markov's inequality, we have that 
\begin{align*}
& \Prb{\exists g\in\left\{\frac{(f-\bar f)(f+\bar f-2\hat\psi_k)}{4c_5}:f\in\Fcal\right\}~:~\PP g^2\leq t^2,~(\PP-\hat\PP_k)g\geq t^2/(4c_5) \mid \{Z_i:i\in I_k^C\}}\\
& \leq 4c_5J(1+c_5J/(\sqrt{n_k}t^2))/t^2,
\end{align*}
where $J=t+\int_0^t\sqrt{\epsilon^{-r}}d\epsilon\leq t+\frac{2}{2-r}t^{1-r/2}$. Thus, $\mathrm{Error}_A=O_p(n^{-1/(2+r)})$.
\end{proof}

\subsection{Proof of Theorem \ref{cdte-rates}}
\begin{proof}
Using the stability conditions for the estimator $\widehat{\EE}_n$ outlined in Theorem \ref{cdte-rates}, we can immediately apply Theorem 1 from \cite{kennedy2020optimal} (as appears in the v2 preprint on arXiv) with the cross-fitted nuisances to obtain:
\begin{align*}
     \left\|(\widehat{\kcdte})-(\kcdte)\right\| & \lesssim  \left\|(\widetilde{\kcdte})-(\kcdte)\right\| \\
     & \quad + \sum_{k=1}^K\left\|\EE[\psi(Z, \widehat{e}\s k, \widehat{\alpha}\s k, \widehat{\nu}\s k)\mid X=x]-\EE[\psi(Z, e^*, \alpha^*, \nu^*)\mid X=x]\right\|\\
     & \lesssim  \left\|(\widetilde{\kcdte})-(\kcdte)\right\|  + \mathcal{E}
\end{align*}
where the last inequality was obtained by applying Theorem \ref{thm:condneym} to nuisance sets $(\widehat{e}\s k, \widehat{\alpha}\s k, \widehat{\nu}\s k)$ in each fold $k\in \overline{1,K}$.
\end{proof}

\subsection{Proof of Theorem \ref{thm:proj-normal}}
\begin{proof} 

Let $\Ical_k = \{i: i = k-1\;(\text{mod}\; K)\}$ be the data indices in the $k$-th data fold, and let $\widehat\EE_k f(Z)= \frac{1}{|\Ical_k|}\sum_{i\in \Ical_k}f(Z_i)$ be the empirical expectation of data with indices in $\Ical_k$. The linear regression parameters $\gamma^*, \widetilde{\gamma}$ and $\widehat{\gamma}$ are given by:
\begin{align*}
    \gamma^* &= \EE[\Xpi\Xpi^T]^{-1}\EE[\kcdte(X)\Xpi]\\
    & = \EE[\Xpi \Xpi^T]^{-1}\EE[\psi(Z, e^*, \alpha^*, \nu^*)\Xpi] \tag{consistency and iterated expectations}\\
    \widetilde{\gamma} &= \widehat\E_n[\Xpi\Xpi^T]^{-1}\widehat\E_n[\psi(Z, e^*, \alpha^*, \nu^*)\Xpi]\\
    & = \widehat\E_n[\Xpi\Xpi^T]^{-1}\frac{1}{K}\sum_{k=1}^K\widehat\EE_k[\psi(Z, e^*, \alpha^*, \nu^*)\Xpi]\\
    \widehat{\gamma} &= \widehat\E_n[\Xpi\Xpi^T]^{-1}
    \frac{1}{K}\sum_{k=1}^K\widehat\EE_k[\psi(Z, \widehat{e}\s k, \s k, \widehat{\nu}\s k)\Xpi]
\end{align*}
We note that we can rewrite $\sqrt{n}(\widehat{\gamma}-\gamma^*)$ as:
\begin{equation}\label{beta-decomp}
    \sqrt{n}(\widehat{\gamma}-\gamma^*)=\sqrt{n}(\widehat{\gamma}-\widetilde{\gamma}) + \underbrace{\sqrt{n}(\widetilde{\gamma}-\gamma^*)}_{\rightsquigarrow \Ncal(0, \Sigma^*)}
\end{equation}
The second term converges in distribution to the desired $\Ncal(0, \Sigma^*)$. Thus, it suffices to show that the first term is $o_p(1)$. We decompose the first term into two components and study them separately:
\begin{align}\label{beta-hat-decomp}
    & \sqrt{n}(\widehat{\gamma}-\widetilde{\gamma}) \\
    & = \sqrt{n} \widehat\E_n[\Xpi\Xpi^T]^{-1}\frac{1}{K}\sum_{k=1}^K\Big(\widehat\EE_k[\psi(Z, \widehat{e}\s k, \widehat{\alpha}\s k, \widehat{\nu}\s k)\Xpi] - \widehat\EE_k[\psi(Z, e^*, \alpha^*, \nu)\Xpi]\Big) \tag*{}\\
    &= \sqrt{n} \widehat\E_n[\Xpi\Xpi^T]^{-1}\frac{1}{K}\sum_{k=1}^K\underbrace{\left(
    \EE[\psi(Z, \widehat{e}\s k, \widehat{\alpha}\s k, \widehat{\nu}\s k)\Xpi] - \EE[\psi(Z, e^*, \alpha^*, \nu^*)\Xpi]\right)}_{\Lambda_{1, k}} \tag*{}\\
    & \quad + \sqrt{n} \widehat\E_n[\Xpi\Xpi^T]^{-1}\frac{1}{K}\sum_{k=1}^K\underbrace{(\widehat\EE_k - \EE)\left [(\psi(Z, \widehat{e}\s k, \widehat{\alpha}\s k, \widehat{\nu}\s k - \psi(Z, e^*, \alpha^*, \nu^*))\Xpi\right]}_{\Lambda_{2, k}} \tag*{}
\end{align}
We now show that $\Lambda_{1,k}$ and $\Lambda_{2,k}$ in the equation above are both $o_p(1/\sqrt{n})$. Let $\Sigma_{\phi}=\EE[\Xpi \Xpi^T]$. Then, each element of  $\Lambda_{1,k}$ is given by:
\begin{align*}
    (\Lambda_{1,k})_i &=  \EE[\psi(Z, \widehat{e}\s k, \widehat{\alpha}\s k, \widehat{\nu}\s k)(\Xpi)_i] - \EE[\psi(Z, e^*, \alpha^*, \nu^*)\Xpi_i]\\
    & = \EE[\EE[\psi(Z, \widehat{e}\s k, \widehat{\alpha}\s k, \widehat{\nu}\s k) - \psi(Z, e^*, \alpha^*, \nu^*)\mid X]\Xpi_i]\\
\Rightarrow  |(\Lambda_{1,k})_i| & \leq \|\EE[\psi(Z, \widehat{e}\s k, \widehat{\alpha}\s k, \widehat{\nu}\s k) - \psi(Z, e^*, \alpha^*, \nu^*)\mid X]\| \;\|\Xpi_i\| \tag{Cauchy-Schwartz}\\
& \lesssim  \Bigg\{\sum_{a=1}^1 \Bigg( \|\widehat{\kappa}\s k_a-\kappa_a^*\| \; \|\widehat{e}\s k-e^*\|
    +\sum_{i=1}^{m+1}\sum_{j=1}^{m+1}G_{ij}\|\widehat{\alpha}\s k_{a,i} - \alpha_{a, i}^*\| \;\|\widehat{\nu}\s k_{a,j} - \nu_{a, j}^*\|
    \\
    & \qquad \quad + \sum_{i=1}^{m+1}\sum_{j=1}^{m+1}H_{ij}\|\widehat{\nu}\s k_{a,i} - \nu^*_{a, i}\| \;\|\widehat{\nu}\s k_{a,j} - \nu_{a, j}^*\|
    \Bigg)\Bigg\} \sqrt{(\Sigma_{\phi})_{ii}} \tag{Thm. \ref{thm:condneym}}
\end{align*}
By the theorem's assumptions, we have that $\Lambda_{1,k}$ is $o_p(1/\sqrt{n})$, as desired.
We now see how we can control $\Lambda_{2,k}$. By Chebyshev's inequality, we have that $\Lambda_{2,k}$ is
\begin{equation}\label{eq:inference-lambda2}
    O_p\left(n^{-1/2}\sum_{i=1}^p \EE[(\psi(Z, \widehat{e}\s k, \widehat{\alpha}\s k, \widehat{\nu}\s k) - \psi(Z, e^*, \alpha^*, \nu^*))^2(\Xpi_{i})^2]^{1/2}\right)
\end{equation}
where we leveraged that $|\Ical_k|\simeq n/K$ and $K$ is a fixed integer constant that doesn't depend on $n$. The sum in the expression above further reduces to:
\begin{align*}
    & \sum_{i=1}^p \EE[(\psi(Z, \widehat{e}\s k, \widehat{\alpha}\s k, \widehat{\nu}\s k) - \psi(Z, e^*, \alpha^*, \nu^*))^2(\Xpi)_{i}^2]^{1/2}\\
    & \qquad \leq \sum_{i=1}^p \|\psi(Z, \widehat{e}\s k, \widehat{\alpha}\s k, \widehat{\nu}\s k) - \psi(Z, e^*, \alpha^*, \nu^*)\|\|\Xpi_{i}\|\\
    & \qquad = \sum_{i=1}^p \sqrt{(\Sigma_{\phi})_{ii}} \;\|\psi(Z, \widehat{e}\s k, \widehat{\alpha}\s k, \widehat{\nu}\s k) - \psi(Z, e^*, \alpha^*, \nu^*)\|
\end{align*}
Thus, it remains to study the convergence of $\|\psi(Z, \widehat{e}\s k, \widehat{\alpha}\s k, \widehat{\nu}\s k) - \psi(Z, e^*, \alpha^*, \nu^*)\|$. We note that: 
\begin{align*}
    \|\psi(Z, \widehat{e}\s k, \widehat{\alpha}\s k, \widehat{\nu}\s k) - \psi(Z, e^*, \alpha^*, \nu^*)\| & = \sum_{a=0}^1 \|\psi^a(Z, \widehat{e}\s k, \widehat{\alpha}\s k, \widehat{\nu}\s k) - \psi^a(Z, e^*, \alpha^*, \nu^*)\|\\
    & \leq \sum_{a=0}^1\Big(\|\psi^a(Z, \widehat{e}\s k, \widehat{\alpha}\s k, \widehat{\nu}\s k) - \psi^a(Z, e^*, \widehat{\alpha}\s k, \widehat{\nu}\s k)\|\\
    & \qquad + \|\psi^a(Z, e^*, \widehat{\alpha}\s k, \widehat{\nu}\s k) - \psi^a(Z, e^*, \alpha^*, \nu^*)\|\Big) \tag{Cauchy-Schwartz}
\end{align*}
where the $\psi^a$'s are defined in the proof of Theorem \ref{thm:condneym}. Without loss of generality, we study the convergence of $\psi^{1}$:
\begin{align*}
    & \|\psi^1(Z, \widehat{e}\s k, \widehat{\alpha}\s k, \widehat{\nu}\s k)- \psi^1(Z, e^*, \widehat{\alpha}\s k, \widehat{\nu}\s k)\| \\
    & \qquad = \left\|\frac{A(\widehat{e}\s k (X)-e^*(X))}{e^*(X)\widehat{e}\s k (X)}\sum_{i=1}^{m+1} \widehat{\alpha}\s k_{1, i}(X, \widehat{\nu}\s k_1)\rho_i(Y, \widehat{\nu}\s k_1)\right\|\\
    & \qquad \leq \frac{(m+1)c_4c_5}{c_1^2}\|\widehat{e}\s k-e^*\| \tag{from boundedness assumptions}\\
    & \qquad \lesssim \|\widehat{e}\s k-e^*\|\\
    & \|\psi^1(Z, e^*, \widehat{\alpha}\s k, \widehat{\nu}\s k) - \psi^1(Z, e^*, \alpha^*, \nu^*)\|\\
    & \qquad = \Bigg\|
                            \widehat{\kappa}\s k_1 (X) - \kappa_1^*(X) - \frac{A}{e^*(X)}\Bigg(\sum_{i=1}^{m+1} \widehat{\alpha}\s k_{1, i} (X, \widehat{\nu}\s k_1)\rho_i(Y, \widehat{\nu}\s k_1)  - \sum_{i=1}^{m+1} \alpha^*_{1, i}(X, \nu_1^*)\rho_i(Y, \nu_1^*)\Bigg)
                         \Bigg\|\\
    & \qquad \leq \|\widehat{\kappa}\s k_1-\kappa_1^*\| + \frac{c_5}{c_1}\sum_{i=1}^m \|\widehat{\alpha}\s k_{1, i} - \alpha^*_{1, i}\|\\
    & \qquad \lesssim \|\widehat{\kappa}\s k_1-\kappa_1^*\| + \sum_{i=1}^m \|\widehat{\alpha}\s k_{1, i} - \alpha^*_{1, i}\|
\end{align*}
Therefore:
\begin{align*}
    \|\psi(Z, \widehat{e}\s k, \widehat{\alpha}\s k, \widehat{\nu}\s k) - \psi(Z, e^*, \alpha^*, \nu^*)\|\lesssim \|\widehat{e}\s k-e^*\| + \sum_{a=0}^1 \left(\|\kappa\s k_a-\kappa^*_a\| + \sum_{i=1}^m \|\widehat{\alpha}\s k_{a, i} - \alpha^*_{a, i}\|\right)
\end{align*}
Given our assumptions on the nuisance convergence rates, the term above is $o_p(1)$. Putting everything together, we obtain that $\Lambda_{1,k}=o_p(1/\sqrt{n})$ and $\Lambda_{1,k}+\Lambda_{2,k}=o_p(1/\sqrt{n})$. Going back to Eq. \ref{beta-hat-decomp}, we have that:
\begin{align*}
    \sqrt{n}(\widehat{\gamma}-\widetilde{\gamma}) &= \sqrt{n}  \widehat\E_n[\Xpi\Xpi^T]^{-1}\frac{1}{K}\sum_{k=1}^K(\Lambda_{1,k} + \Lambda_{2,k})
\end{align*}
Using the continuous mapping theorem, we have that $\widehat\E_n[\Xpi\Xpi^T]^{-1}\overset{P}{\to} \Sigma_\phi^{-1}$. Finally, by Slutsky's theorem and the fact that $K$ does not grow with $n$, we obtain that $\sqrt{n}(\widehat{\gamma}-\widetilde{\gamma})$ is $o_p(1)$ and therefore:
\begin{equation*}
    \sqrt{n}(\widehat{\gamma}-\gamma^*) \rightsquigarrow \Ncal(0, \Sigma^*)
\end{equation*}
as desired.

\end{proof}

\section{APPLICATIONS OF THEOREM \ref{thm:condneym}}\label{sec:thm-app}

\subsection{Pseudo-outcomes and Rates for CQTE}

\begin{corollary}[Pseudo-outcomes and rates for CQTE]\label{cor:cqte-rates} Assume the distribution $F$ is continuous. Then, the pseudo-outcome for the conditional quantile treatment at level $\tau$, $\text{CQTE}(x;\tau)$, is given by: 
\begin{equation*}
    \psi^{\text{CQTE}}(Z, e, q, f) = q_1(X; \tau) - q_0(X;\tau) + \frac{A-e(X)}{e(X)(1-e(X))}\frac{1}{f_A(x)}\left(\tau-\II[Y\leq q_A(x;\tau)]\right)
\end{equation*}
where $f_a(x)=f_{Y\mid X=x, A=a}(q_a(x;\tau))$ is the conditional density of $Y$ evaluated at the conditional quantile.  Furthermore, if the conditions of Assumption \ref{cdte-assum} are satisfied, the oracle rate deviation is given by:
\begin{equation*}
    \mathcal{E} \leq \sum_{k=1}^K\sum_{a=0}^1\|\widehat{q}\s k_a-q^*_a\|\left(\|\widehat{e}\s k -e^*\|+\left\|\frac{1}{\widehat{f}\s k_a}-\frac{1}{f^*_a}\right\| + \|\widehat{q}\s k_a -q^*_a\|\right)
\end{equation*}
\end{corollary}
\begin{remark}
We note that by Assumption \ref{cdte-assum}, $f_a$ and $\widehat{f}_a$ are bounded away from $0$, and thus the rate deviation can be written as $\sum_{k=1}^K\sum_{a=0}^1\|\widehat{q}\s k_a-q^*_a\|\left(\|\widehat{e}-e^*\|+\|\widehat{f}\s k_a-f^*_a\| + \|\widehat{q}\s k_a-q^*_a\|\right)$. When considering medians ($\alpha=0.5$), this result reduces to the result in \cite{leqi2021median} for median effects.
\end{remark}

\begin{proof}

For continuous CDFs, the conditional quantile at level $\tau$ is the solution to the moment:
\begin{equation*}
\EE[\tau-\II[Y\leq q_a(x;\tau)]\mid X=x, A=a]=0
\end{equation*}

The Jacobian is given by $J_a^*(X) = D_{q_a}\{\EE[\tau-\II[Y\leq q_a(X;\tau)]\mid X=x, A=a]\} \big\lvert_{q_a=q_a^*} = -f^*_a(x)$ where $f_a(x)=f_{Y\mid X=x, A=a}(q_a(x;\tau))$ is the conditional density evaluated at the conditional quantile. By Definition \ref{pseudo-def}, $\alpha^*_a(X)=-1/f^*_a(X)$ and the pseudo-outcome is given by:
\begin{equation*}
    \psi^{\text{CQTE}}(Z, e, q, f) = q_1(X; \tau) - q_0(X;\tau) + \frac{A-e(X)}{e(X)(1-e(X))}\frac{1}{f_A(x)}\left(\tau-\II[Y\leq q_A(x;\tau)]\right)
\end{equation*}
We apply Theorem \ref{thm:condneym} with $\alpha_a=1/f_a$, $\nu_a = q_a$ and obtain that the oracle deviation is given by:
\begin{equation*}
     \mathcal{E} \leq \sum_{k=1}^K\sum_{a=0}^1\|\widehat{q}\s k_a-q^*_a\|\left(\|\widehat{e}\s k -e^*\|+\left\|\frac{1}{\widehat{f}\s k_a}-\frac{1}{f^*_a}\right\| + \|\widehat{q}\s k_a -q^*_a\|\right)
\end{equation*}
as desired.

\end{proof}

\subsection{Pseudo-outcomes and Rates for CSQTE}

\begin{corollary}[Pseudo-outcomes and rates for CSQTE]\label{cor:csqte-rates} Assume the distribution $F$ is continuous. Then, the pseudo-outcome for the conditional super-quantile treatment at level $\tau$, $\text{CSQTE}(x;\tau)$, is given by: 
\begin{align*}
    \psi^{\text{CSQTE}}(Z, e, \mu, q) & = \mu_1(X; \tau) - \mu_0(X;\tau)\\
    & \quad + \frac{A-e(X)}{e(X)(1-e(X))} \Big(
     q_A(X;\tau) + \frac{1}{1-\tau}(Y-q_A(X;\tau))\II[Y\geq q_A(X;\tau)] - \mu_A(X;\tau)\Big)
\end{align*}
Furthermore, if the conditions of Assumption \ref{cdte-assum} are satisfied, the oracle rate deviation is given by:
\begin{equation*}
    \mathcal{E} \leq \sum_{k=1}^K\sum_{a=0}^1\left(\|\widehat{\mu}\s k_a-\mu^*_a\|\|\widehat{e}\s k -e^*\|+\|\widehat{q}\s k_a-q^*_a\|^2\right)
\end{equation*}
\end{corollary}

\begin{proof}
From Example \ref{sec:csqte}, we have that for continuous CDFs, the conditional super-quantile at level $\tau$ is the solution to the conditional moments:
\begin{align*}
    & \EE[(1-\tau)^{-1}Y\II[Y\geq q_a(x;\tau)] - \mu_a(x;\tau)\mid X=x, A=a]=0\\
    & \EE[\tau-\II[Y\leq q_a(x;\tau)]\mid X=x, A=a]=0
\end{align*}
Thus, the Jacobian $J_a^*(X)$ and its inverse if given by:
\begin{equation*}
    J_a^*(X) = \begin{pmatrix}
    -1 & - (1-\tau)^{-1}q_a(X;\tau)f_a(X)\\
    0 & - f_a(X)
    \end{pmatrix}, \quad 
(J_a^*(X))^{-1} = \begin{pmatrix}
    -1 & (1-\tau)^{-1}q_a(X;\tau)\\
    0 & - 1/f_a(X)
    \end{pmatrix}
\end{equation*}
where $f_a(x)=f_{Y\mid X=x, A=a}(q_a(x;\tau))$ is the conditional density evaluated at the conditional quantile. By Definition \ref{pseudo-def} with $\alpha_a = (-1, -(1-\tau)^{-1}q_a(X;\tau))$, $\nu_a = (\mu_a, q_a)$, the pseudo-outcome is given by:
\begin{align*}
    \psi^{\text{CSQTE}}(Z, e, \mu, q) & = \mu_1(X; \tau) - \mu_0(X;\tau)+\frac{A-e(X)}{e(X)(1-e(X))}\frac{1}{1-\tau} \Big(\\
    & \quad Y\II[Y\geq q_A(X;\tau)] - (1-\tau)\mu_A(X;\tau)
    -q_A(X;\tau)\left(\tau - \II[Y\leq q_A(X;\tau)]\right)\Big)\\
    & = \mu_1(X; \tau) - \mu_0(X;\tau)\\
    & \quad + \frac{A-e(X)}{e(X)(1-e(X))} \Big(
     q_A(X;\tau) + \frac{1}{1-\tau}(Y-q_A(X;\tau))\II[Y\geq q_A(X;\tau)] - \mu_A(X;\tau)\Big)
\end{align*}
We now apply Theorem \ref{thm:condneym}. We first note that the binary matrices $G$ and $H$ are given by:
\begin{equation*}
    G = \begin{pmatrix}
    1 & 1\\
    0 & 1
    \end{pmatrix},\quad
    H = \begin{pmatrix}
    0 & 0\\
    0 & 1
    \end{pmatrix}
\end{equation*}
Moreover, $\|\widehat{\alpha}\s k_{a, 1} - \alpha^*_{a, 1}\|=0$ since $\alpha_{a, 1}=-1$ is a constant. Therefore, in the cross-products $\|\widehat{\alpha}\s k_{a, i} - \alpha^*_{a, i}\|\|\widehat{\nu}\s k_{a, j} - \nu^*_{a, j}\|$, the only surviving term is $\|\widehat{\alpha}\s k_{a, 2} - \alpha^*_{a, 2}\|\|\widehat{\nu}\s k_{a, 2} - \nu^*_{a, 2}\|=\|\widehat{q}\s k_{a} - q^*_a\|^2$ due to $G_{21}=0$.  Furthermore, given $H$, the only surviving term in the cross-products $\|\widehat{\nu}\s k_{a, i} - \nu^*_{a, i}\|\|\widehat{\nu}\s k_{a, j} - \nu^*_{a, j}\|$ is $\|\widehat{\nu}\s k_{a, 2} - \nu^*_{a, 2}\|^2=\|\widehat{q}\s k_{a} - q^*_a\|^2$. Thus, the oracle deviation is given by: 
\begin{equation*}
    \mathcal{E} \leq \sum_{k=1}^K\sum_{a=0}^1\left(\|\widehat{\mu}\s k_a-\mu^*_a\|\|\widehat{e}\s k -e^*\|+\|\widehat{q}\s k_a-q^*_a\|^2\right)
\end{equation*}
as desired.
\end{proof}

\subsection{Pseudo-outcomes and Rates for C$f$RTE}\label{sec:cklrte-rates}

\begin{corollary}[Pseudo-outcomes and rates for C$f$RTE]\label{cor:cfrte-rates} The conditional $f$-risk treatment effect at level $\delta$, $\text{C}f\text{RTE}(x;\delta)$, is given by: 
\begin{align*}
    \psi^{\text{C}f\text{RTE}}(Z, e, R, \beta, \lambda) = R^f_1(X; \delta) - R^f_0(X;\delta)+\frac{A-e(X)}{e(X)(1-e(X))}(m(Y, \beta_A, \lambda_A;\delta) - R^f_A(X; \delta))
\end{align*}
Furthermore, if the conditions of Assumption \ref{cdte-assum} are satisfied, the oracle rate deviation is given by:
\begin{align*}
    & \mathcal{E} \leq \sum_{k=1}^K\sum_{a=0}^1\Big(\|\widehat{R}^{f, (k)}_a-R^{f,*}_a\|\|\widehat{e}\s k-e^*\| + \|\widehat{\beta}\s k_a-\beta_a^*\|^2 + \|\widehat{\lambda}\s k_a-\lambda_a^*\|^2 + \|\widehat{\beta}\s k_a-\beta_a^*\|\|\widehat{\lambda}\s k_a-\lambda_a^*\|\Big)
\end{align*}
\end{corollary}

\begin{proof}
From Example \ref{sec:cfrte}, we have that the C$f$RTE is identified by the following conditional moments:
\begin{align*}
    & \EE[m(Z, \beta_a, \lambda_a; \delta) - R^f_a(x;\delta)\mid X=x, A=a] = 0 \\
    & \EE\left[\frac{\partial}{\partial \beta_a}m(Z, \beta_a, \lambda_a; \delta) \bigg\lvert X=x, A=a\right] = 0\\
    & \EE\left[\frac{\partial}{\partial \lambda_a}m(Z, \beta_a, \lambda_a; \delta)\bigg\lvert X=x, A=a\right] = 0
\end{align*}
where it is understood that $\beta_a$ and $\lambda_a$ are functions of $X$. Thus, the Jacobian $J_a^*(X)$ and its inverse if given by:
\begin{equation*}
J_a^*(X) =
\begin{pmatrix}
     -1 & 0 & 0\\
    0 & \scriptstyle-\EE\left[\frac{(Y-\lambda_a(X))^2}{\beta_a(X)^3}(f^*)''\left(\frac{Y-\lambda_1(X)}{\beta_a(X)}\right)\big| X=x, A=1\right] & \scriptstyle \EE\left[\frac{Y-\lambda_a(X)}{\beta_a(X)^2}(f^*)''\left(\frac{Y-\lambda_a(X)}{\beta_a(X)}\right)\big| X=x, A=1\right] \\
    0 & \scriptstyle \EE\left[\frac{Y-\lambda_a(X)}{\beta_a(X)^2}(f^*)''\left(\frac{Y-\lambda_a(X)}{\beta_a(X)}\right)\big| X=x, A=1\right] & \scriptstyle -\EE\left[\frac{1}{\beta_a(X)}(f^*)''\left(\frac{Y-\lambda_a(X)}{\beta_a(X)}\right)\big| X=x, A=1\right]
    \end{pmatrix}
\end{equation*}

\begin{equation*}
(J_a^*(X))^{-1} = \begin{pmatrix}
    -1 & 0 & 0\\
    0 & ... & ...\\
    0 & ... & ...
    \end{pmatrix}
\end{equation*}
 By Definition \ref{pseudo-def} with $\alpha_a = (-1, 0, 0)$, $\nu_a = (R^f_a, \beta_a, \lambda_a)$, the pseudo-outcome is given by:
\begin{align*}
    \psi^{\text{C}f\text{RTE}}(Z, e, R, \beta, \lambda) = R^f_1(X; \delta) - R^f_0(X;\delta)+\frac{A-e(X)}{e(X)(1-e(X))}(m(Y, \beta_A, \lambda_A;\delta) - R^f_A(X; \delta))
\end{align*}
We now apply Theorem \ref{thm:condneym}. We first note that the binary matrices $G$ and $H$ are given by:
\begin{equation*}
     G = \begin{pmatrix}
    1 & 0 & 0\\
    0 & 1 & 1\\
    0 & 1 & 1
    \end{pmatrix},\quad
    H = \begin{pmatrix}
    0 & 0 & 0\\
    0 & 1 & 1\\
    0 & 1 & 1
    \end{pmatrix}
\end{equation*}
Moreover, $\|\widehat{\alpha}\s k_{a, i} - \alpha^*_{a, i}\|=0$ since the $\alpha_{a, i}$'s are constants. Therefore, the cross-products $\|\widehat{\alpha}\s k_{a, i} - \alpha^*_{a, i}\|\|\widehat{\nu}\s k_{a, j} - \nu^*_{a, j}\|$ vanish. Furthermore, given $H$, the only surviving terms in the cross-products $\|\widehat{\nu}\s k_{a, i} - \nu^*_{a, i}\|\|\widehat{\nu}\s k_{a, j} - \nu^*_{a, j}\|$ involve only $\beta_a$ and $\lambda_a$. Thus, the oracle deviation is given by:
\begin{align*}
    & \mathcal{E} \leq \sum_{k=1}^K\sum_{a=0}^1\Big(\|\widehat{R}^{f, (k)}_a-R^*_a\|\|\widehat{e}\s k-e^*\| + \|\widehat{\beta}\s k_a-\beta_a^*\|^2 + \|\widehat{\lambda}\s k_a-\lambda_a^*\|^2 + \|\widehat{\beta}\s k_a-\beta_a^*\|\|\widehat{\lambda}\s k_a-\lambda_a^*\|\Big)
\end{align*}\end{proof}

We further provide a specific example of C$f$RTE for $f(x)=x\log x$ which corresponds to the Kullback–Leibler (KL) divergence. The associated risk is known as the entropic value-at-risk (EVaR). We term this conditional $f$-risk treatment effect the CKLRTE. The convex conjugate of $f$ is given by $f^*(x^*)=e^{x^*-1}$. From the first-order optimality conditions, we obtain:
\begin{align}\label{eq:cklrte}
    \beta_a^*(x;\delta) &= {\arg\inf}_{\beta_a(x;\delta)\geq 0} \; \beta_a(x;\delta)\left(
    \log \EE\left[e^{\frac{Y}{\beta_a(x;\delta)}}\big| X=x, A=a\right] + \delta
    \right)\\
    R^{KL}(x;\delta) &= \beta^*_a(x;\delta)\left(
    \log \EE\left[e^{\frac{Y}{\beta^*_a(x;\delta)}}\big| X=x, A=a\right] + \delta
    \right)\tag*{}\\
    \lambda_a^*(x;\delta) &= \beta^*_a(x;\delta)\left(
    \log \EE\left[e^{\frac{Y}{\beta^*_a(x;\delta)}}\big| X=x, A=a\right] - 1
    \right)\tag*{}\\
    & = R^{KL}(x;\delta) - \beta_a^*(x;\delta) (\delta + 1)
\end{align}
Thus, the optimization problem contains only one variable of interest, $\beta_a(x;\delta)$. The pseudo-outcome for CKLRTE at level $\delta$ is then given by:
\begin{align*}
    \psi^{\text{CKLRTE}}(Z, e, R, \beta, \lambda) &= R^{KL}_1(X; \delta) - R^{KL}_0(X;\delta)\\
    & \quad +\frac{A-e(X)}{e(X)(1-e(X))}\left(\delta \beta_A(X;\delta) + \lambda_A(X;\delta) + \beta_A(X;\delta) e^{\frac{Y-\lambda_A(X;\delta)}{\beta_A(X;\delta)}-1} - R^{KL}_A(X; \delta)\right)
\end{align*}

\begin{remark} An interesting problem is the estimation of $\beta^*_a(X;\delta)$. We observe that the empirical analog of \cref{eq:cklrte} requires fitting a large number of regression functions $\widehat{\EE}_n\left[e^{\frac{Y}{\beta_a(x;\delta)}}\big| X=x, A=a\right]$, one for each candidate $\beta_a(x;\delta)$. This can be mitigated by instead learning similarity weights between $x$ and the training data and then replacing $\widehat{\EE}_n\left[e^{\frac{Y}{\beta_a(x;\delta)}}\big| X=x, A=a\right]$ with a weighted sum. We can then use this approximation with any convex optimization method since the argument of \cref{eq:cklrte} is a convex function in $\beta_a$. 
\end{remark}

\section{ADDITIONAL EXPERIMENTAL RESULTS}\label{sec:exp-results}

The replication code is distributed under an MIT license. The results in Section \ref{sec:empirical} were obtained using an Amazon Web Services instance with 32 vCPUs and 64 GiB of RAM. 

\subsection{Simulation Study}

In this section, we provide additional results on simulated data for CQTEs (\cref{sec:cqte}) and C$f$RTE (\cref{sec:cfrte}). Our analysis relies on the same data generating process and setup described in \cref{sec:empirical}. For nuisance and second-stage estimation, we use Python's \texttt{scikit-learn} library \citep{scikit-learn} for logistic and linear regression, random forest regression (RF) and linear quantile regression (LQR). We employ the extension \texttt{sklearn-quantile} for quantile random forest regression (QRF). We use the estimators' default hyperparameters except for the RFs where we set the minimum leaf size to $n/20$ to control overfitting. For each comparison, we run $100$ simulations for each sample size $n=100,200,\ldots,12800$ and evaluate the mean squared error (MSE) over a fixed set of 500 random $X$ values.

\noindent \textbf{CQTE Experiments.} We measure the conditional quantile treatment effect at level $\tau=0.75$. For the given DGP, the true CQTE at level $\tau$ is given by $q_1(X;\tau)-q_0(X;\tau)\simeq 1.14(e^{X_0+X_1}-e^{X_0})$, a heterogenous function of $X$. Learning CQTEs requires estimating three nuisances: the propensity score $e(X)$, the conditional quantile $q_a(X;\tau)$, as well as the density at the conditional quantile, $f_a(X)$. 

We estimate $e(X)$ using logistic regression. Likewise, we learn $f_a(X)$ using the method described in \cref{sec:nuisance-est} with a Gaussian kernel (bandwidth $b=1$) for the estimates $\omega_i$ and a random forest (RF) regressor as the final estimator $\widehat{\EE}_n[\omega\mid X=x,A=a]$. Finally, we consider three methods for estimating $q_a(X;\tau)$. First, we consider a quantile RF (QRF) \citep{meinshausen2006quantile} as the \textit{flexible} learner. Then, we consider a "\textit{misspecified}" model such as the linear quantile regressor (LQR) \citep{koenker2005quantile}. Lastly, we consider an estimator that uses a Gaussian kernel for calculating weights and then computes $\widehat{q}_a$ using the weighted version of the moment in \cref{eq:cqte-moment}. We choose the kernel bandwidth by Silverman’s rule \citep{silverman2018density}. This will be our \textit{slow} estimator as we expect it to suffer from the curse of dimensionality. For the final stage, we use an ordinary least squares model (CQTE+OLS) or a RF (CQTE+RF).

\begin{figure*}[t]
\vspace{-1em}
     \centering
     \begin{subfigure}[b]{0.29\textwidth}
         \centering
         \includegraphics[width=\textwidth]{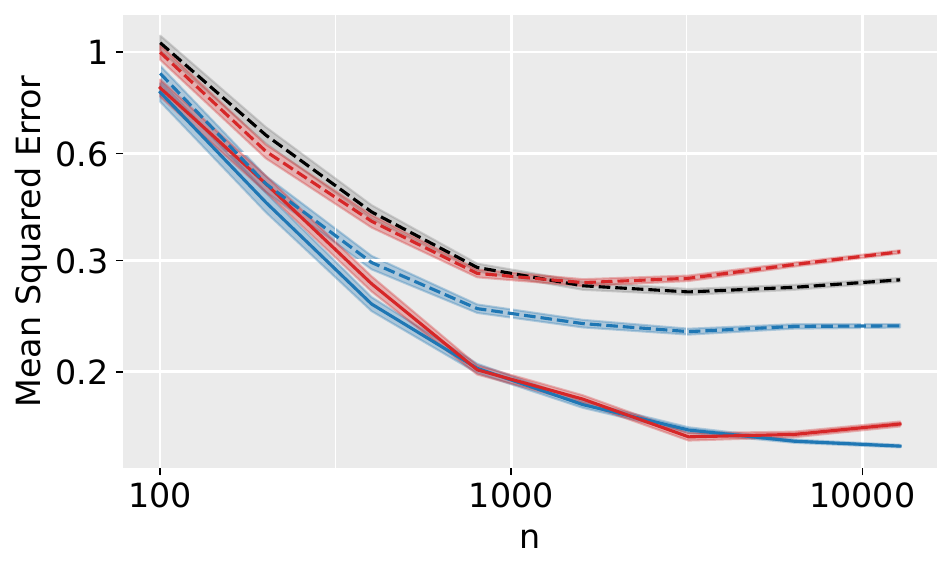}
     \end{subfigure}
     \begin{subfigure}[b]{0.28\textwidth}
         \centering
         \includegraphics[width=\textwidth]{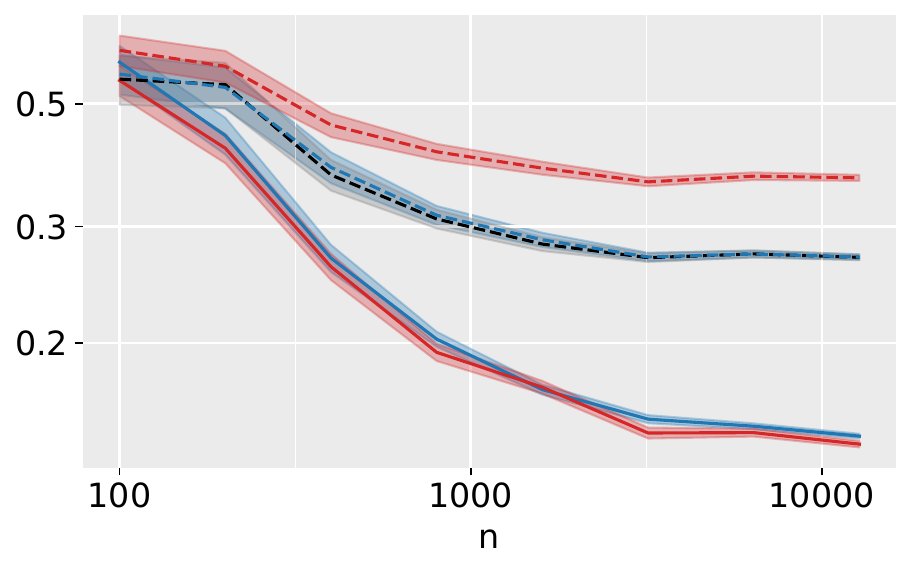}
     \end{subfigure}
     \begin{subfigure}[b]{0.39\textwidth}
         \centering
         \includegraphics[width=\textwidth]{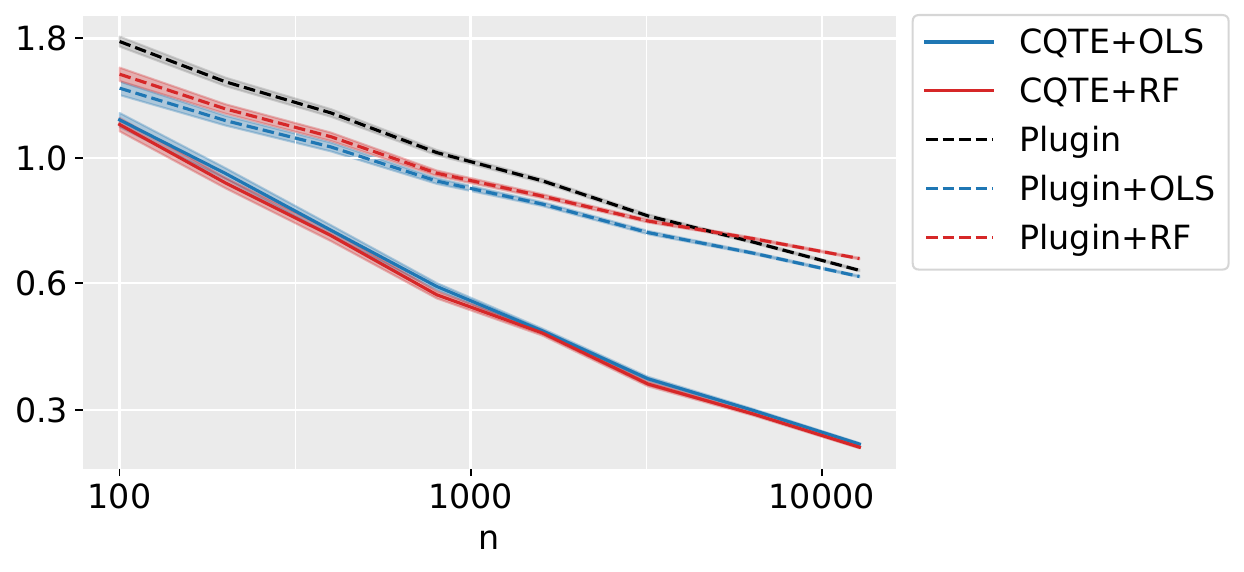}
     \end{subfigure}\\
     \begin{subfigure}[b]{0.29\textwidth}
         \centering
         \includegraphics[width=\textwidth]{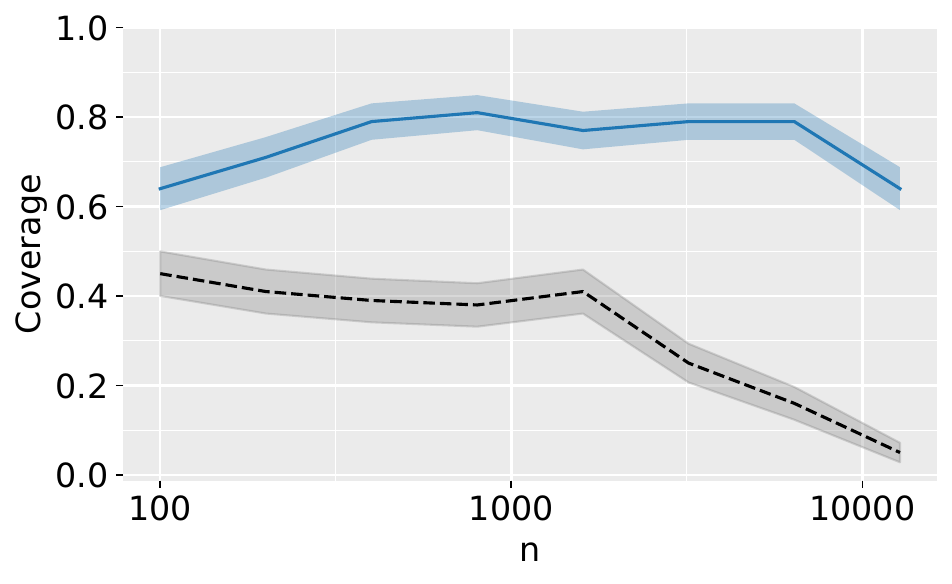}
         \caption{Flexible $\hat{q}$}
     \end{subfigure}
     \begin{subfigure}[b]{0.28\textwidth}
         \centering
         \includegraphics[width=\textwidth]{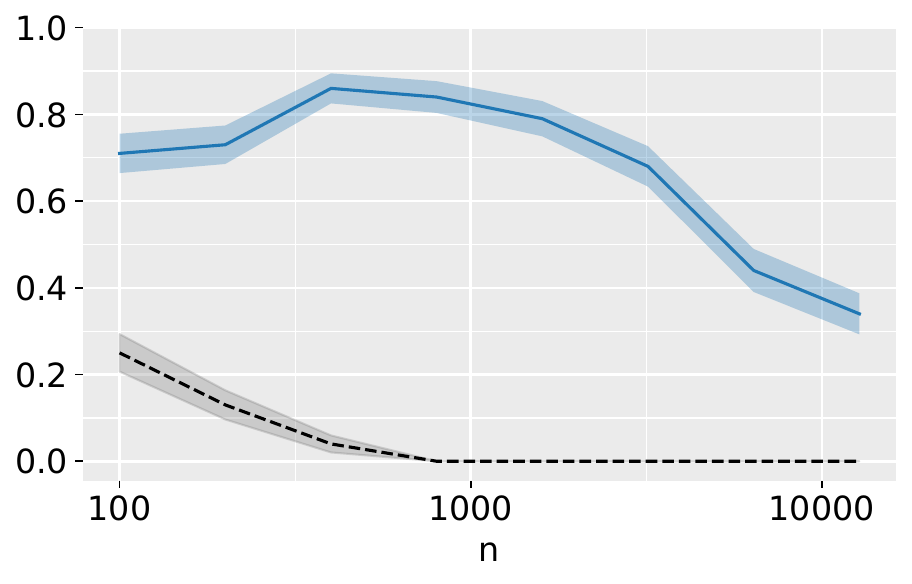}
         \caption{Misspecified $\hat{q}$}
     \end{subfigure}
     \begin{subfigure}[b]{0.39\textwidth}
         \centering
         \includegraphics[width=\textwidth]{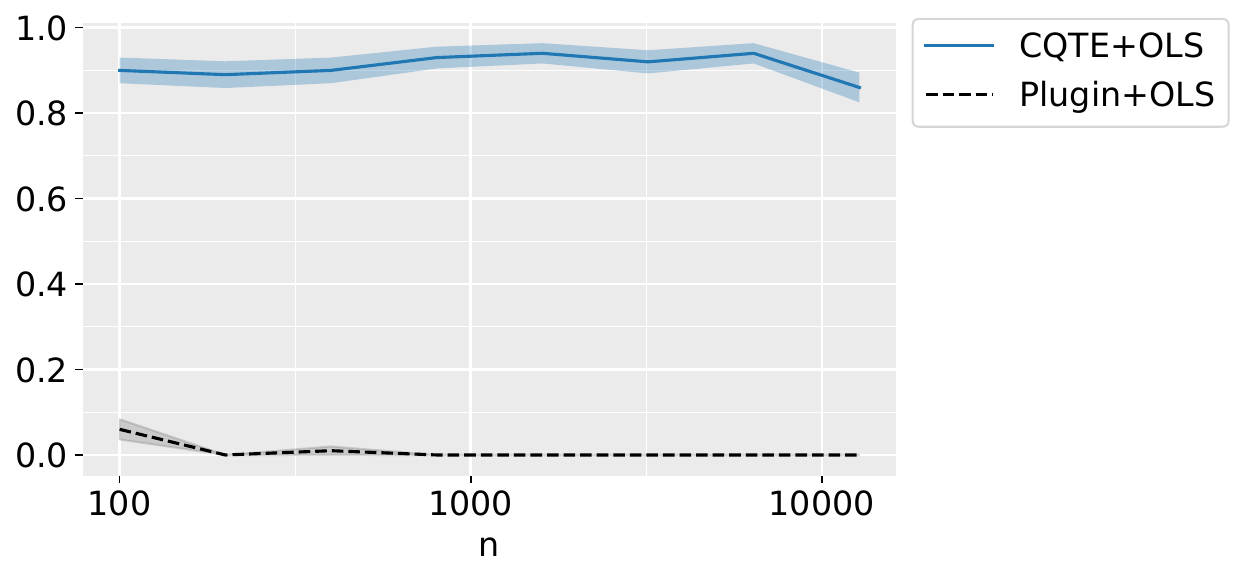}
         \caption{Slow $\hat{q}$}
     \end{subfigure}
        \caption{Mean squared error (MSE) and confidence 95\% interval coverage for different CQTE learners. Shaded regions depict plus/minus one standard error over $100$ simulations.}
        \label{fig:cqte-sims}
        \vspace{-1em}
\end{figure*}

For the performance comparisons, we proceed similarly to \cref{sec:empirical}. Thus, we compare the out-of-sample MSE of the CQTE estimator with that of the plug-in estimator from Eq. \ref{plugin-alg}. We also construct Plugin+OLS and Plugin+RF given by running an additional OLS/RF model on the cross-fitted plug-in predictions. We do so to account for additional smoothing from the last stage regressor. When the second stage algorithm is OLS, we check whether the 95\%-confidence interval OLS returns for the $X_1$ coefficient contains the coefficient from the true projection. The results are shown in  \Cref{fig:cqte-sims}. Our CQTE learner provides uniformly better MSE performance that the plugin counterparts, and the results show this is not just a consequence of the second-stage regression. For inference, we achieve better coverage than the plug-in approaches. However, for the flexible and "misspecified" estimator, the coverage does not reach the desired level of 95\%. This is most likely due to misspecifcation in the conditional density learner. \textit{Note:} we put \textit{misspecified} in quotes since the MSE of this estimator is better than that of the \textit{flexible} estimator. We attribute this to the fact that for small $t$ we can approximate $e^t\simeq 1+t$ which makes the true CQTE close to a linear approximation.

\noindent \textbf{C$f$RTE Experiments.} We now turn to estimating the conditional $f$-risk treatment effects when we set the $f$-divergence to be the KL divergence. This $f$-risk measures the difference between conditional entropic values-at-risk (EVaRs). We call this risk treatment effect the CKLRTE (see \cref{sec:cklrte-rates}). We now wish to measure the CKLRTE at level $\tau=0.75$ ($\delta=-\log(1-\tau)$). We note that our DGP does not admit EVaRs as the moment generating function of a lognormal distribution diverges. Thus, we modify the DGP to truncate any values above the 99$^\text{th}$ conditional quantile. For the truncated DGP, the true CKLRTE at level $\delta$ is given by $R^{KL}_1(X;\delta)-R^{KL}_0(X;\delta)\simeq 1.42(e^{X_0+X_1}-e^{X_0})$, a heterogeneous function of $X$. Learning CKLRTEs requires estimating the following nuisances: the propensity score $e(X)$, the conditional EVaR $R^{KL}_a(X;\delta)$ and the $\beta_a(X; \delta)$ optimization parameter.

As before, we estimate $e(X)$ using logistic regression. We learn $R^{KL}_a(X;\delta)$ and the $\beta_a(X; \delta)$ optimization parameter jointly via the procedure described in  \cref{sec:cklrte-rates}. Namely, we learn the weights necessary for computing a weighted average version of $\widehat{\EE}_n\left[e^{\frac{Y}{\beta_a(x;\delta)}}\big| X=x, A=a\right]$ and we then solve the convex optimization problem in \cref{eq:cklrte}. We use two models to calculate the weights: a random forest (\textit{flexible} learner, most likely to capture complex functions) and a Gaussian kernel (\textit{slow} learner, most likely suffers from the curse of dimensionality). We choose the Gaussian kernel bandwidth by Silverman’s rule \citep{silverman2018density}. For the final stage, we use an OLS model (CKLRTE+OLS) or a RF (CKLRTE+RF).

 Similar to our other experimental benchmarks, we compare the out-of-sample MSE of the CKLRTE estimator with that of the plug-in estimator from Eq. \ref{plugin-alg}. We also construct Plugin+OLS and Plugin+RF given by running an additional OLS/RF model on the cross-fitted plug-in predictions. This will account for any additional smoothing from the last stage regressor. When the second stage algorithm is OLS, we check whether the 95\%-confidence interval for the $X_1$ coefficient contains the coefficient from the true projection. We display the results in  \Cref{fig:cklrte-sims}. Our CKLRTE learner provides uniformly better MSE performance than the plugin counterparts, regardless of the second stage regression. For inference, we achieve good coverage whereas plug-in approaches yield little to no coverage. 

\begin{figure*}[t]
\vspace{-1em}
     \centering
     \begin{subfigure}[b]{0.29\textwidth}
         \centering
         \includegraphics[width=\textwidth]{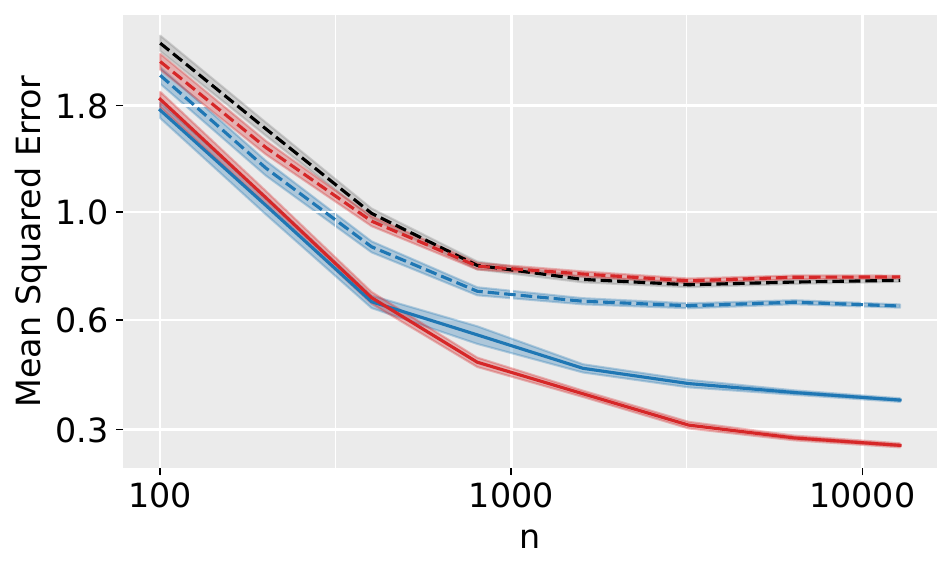}
     \end{subfigure}
     \begin{subfigure}[b]{0.39\textwidth}
         \centering
         \includegraphics[width=\textwidth]{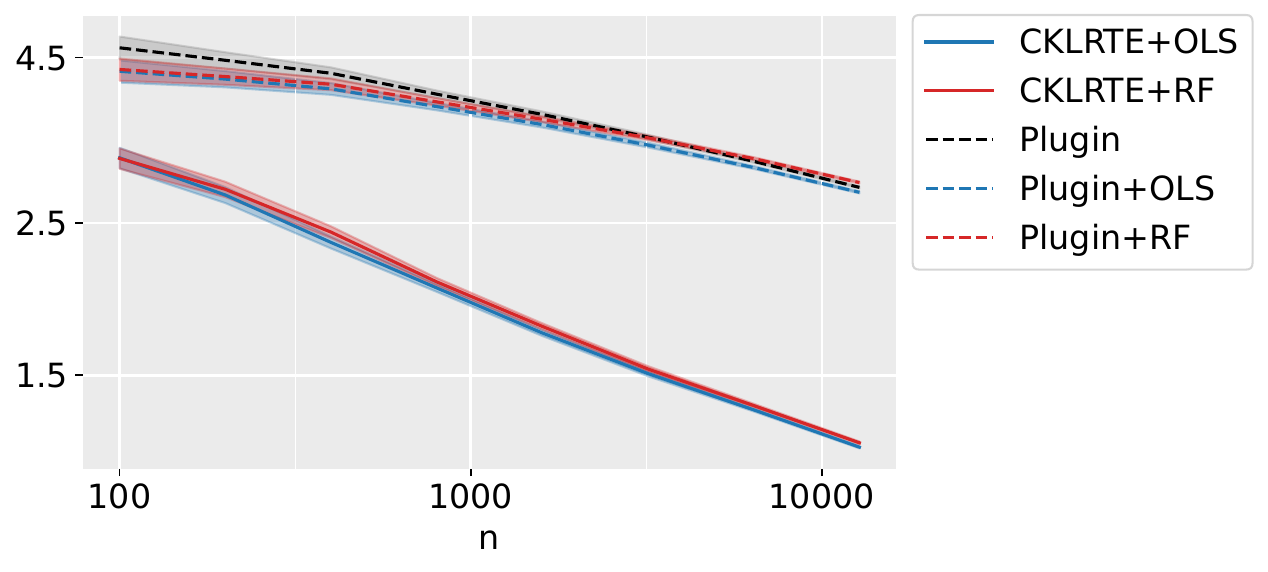}
     \end{subfigure}\\
     \begin{subfigure}[b]{0.29\textwidth}
         \centering
         \includegraphics[width=\textwidth]{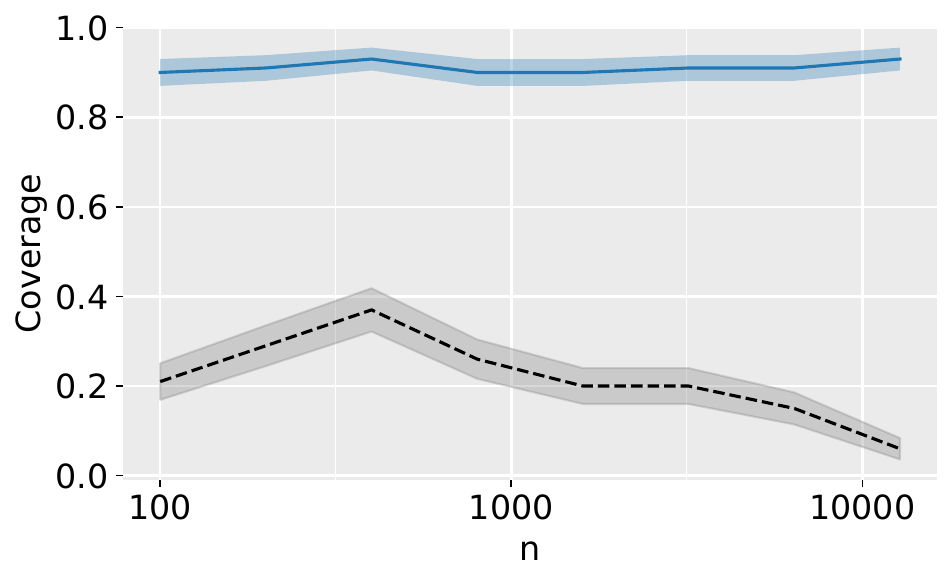}
         \caption{Flexible $\widehat{R}^{KL}$}
     \end{subfigure}
     \begin{subfigure}[b]{0.39\textwidth}
         \centering
         \includegraphics[width=\textwidth]{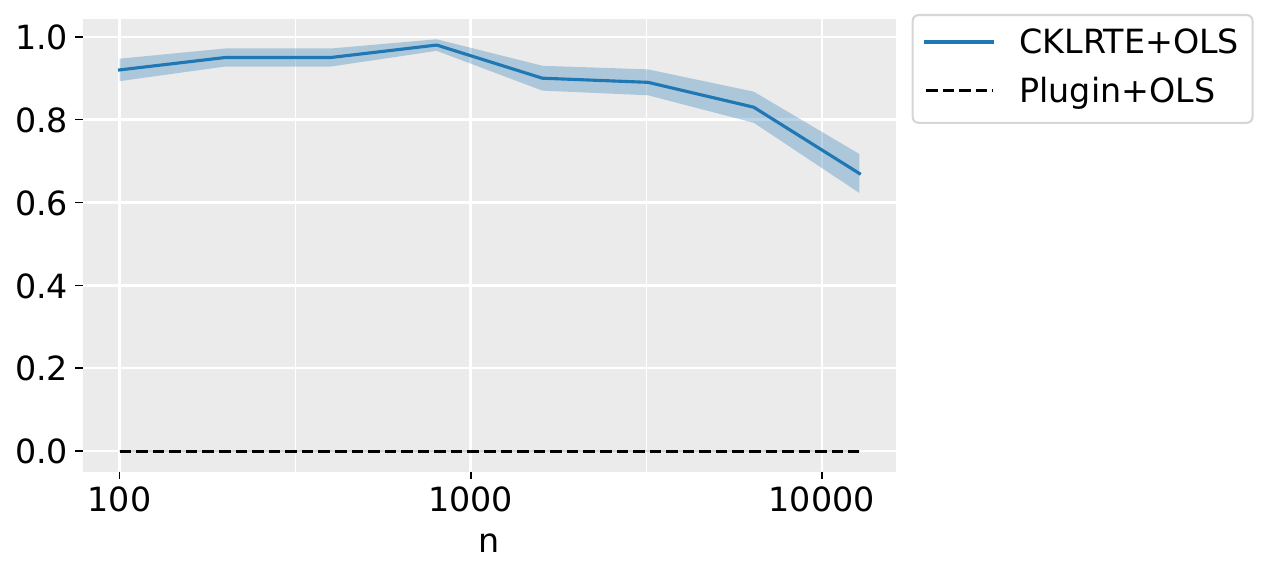}
         \caption{Slow $\widehat{R}^{KL}$}
     \end{subfigure}
        \caption{Mean squared error (MSE) and 95\% confidence interval coverage for different CKLRTE learners. Shaded regions depict plus/minus one standard error over $100$ simulations.}
        \label{fig:cklrte-sims}
\end{figure*}

\subsection{Impact of 401(k) Eligibility on Financial Wealth}

We provide a detailed description the 401(k) dataset features in \cref{tab:401k_feats}. We then compare the feature importances given by the random forest final stage of the CSQTE estimators and the DR-Learner. \Cref{fig:app-401k-rf-feature-importance} shows that the features driving the treatment effects for the three estimators have the same importance profile across tasks. For example, income, age and education are the most important features when determining the conditional average treatment effect, as well as the conditional average effects in the left 25\% tail and right 25\% tail. 

\begin{table}[t]
\centering
\captionof{table}{Features of 401(k) dataset.}
\begin{tabular}{l l l } 
 \hline
 Name & Description & Type  \\ [0.5ex] 
 \hline\hline
age & age & continuous  \\
inc & income & continuous \\
educ & years of completed education & continuous \\
fsize & family size & continuous \\
marr & marital status & binary \\
two\_earn & whether dual-earning household & binary\\
db & defined benefit pension status & binary \\
pira & IRA participation & binary \\
hown & home ownership & binary\\
e401 & 401 (k) eligibility & binary \\
net\_tfa & net financial assets & continuous\\
\hline
\end{tabular}
\label{tab:401k_feats}
\end{table}

\begin{figure}[t]
\centering
    \includegraphics[scale=0.4]{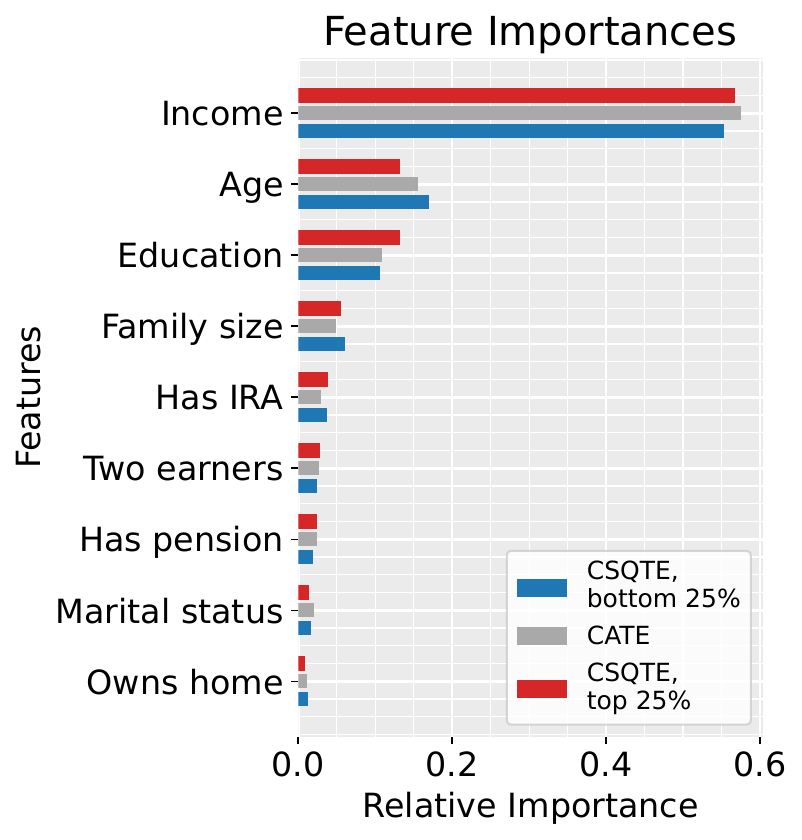}
    \caption{Feature importances given by the final stage algorithm for CSQTE on the bottom 25\% financial asset holders, DR-Learner and CSQTE on the top 25\% financial asset holders}
    \label{fig:app-401k-rf-feature-importance}
    \vspace{-1em}
\end{figure}

\section{PRACTICAL CONSIDERATIONS FOR AND LIMITATIONS OF CDTE ESTIMATION} \label{sec:limitations}

One of the limitations of our work is that the interpretation of our estimands as causal effects only hold when the unconfoundedness assumption holds. Whether it holds or does not, our estimands are differences in distributional statistics between the distributions of $Y\mid X,A=1$ and $Y\mid X,A=0$. If it does hold, this coincides with the differences in distributional statistics between the distributions of $Y(1)\mid X$ and $Y(0)\mid X$. 
While unconfoundedness can usually be guaranteed in experiments barring any non-compliance, there is no way to test for it in observational data. Thus, the onus is on the researcher to determine whether unconfoundedness is plausible in their data and to interpret the estimands carefully. 

Another possible concern in practice is the possibility for biased selection, where the sample may not be representative of the population of interest. To the extent that any unrepresentativeness is explained by $X$ (known as, selection at random), the CDTEs given $X$ will be unaffected. Therefore, we can alleviate this issue in the application of our method by learning CDTEs with a second-stage regression that is a universal approximator (such as nonparametric regression, forests, or deep networks) so it learns the true CDTEs, thus adjusting for $X$ fully.
If there is unrepresentativeness not reflected in $X$ (known as, selection \textit{not} at random), then even the true CDTEs given $X=x$ only reflect effects on the data-generating subpopulation with $X=x$. This may be alleviated by considering richer $X$ so as to minimize the discrepancy and/or by acknowledging possible biases. Finally, whether selection is at random or not, any best-in-class prediction guarantees (such as for simple second-stage regressions like OLS) only hold with respect to the data-generating population, and therefore ``best-performance'' may be unrepresentative of performance in the population of interest. If selection is at random, this is nonetheless easily fixable by reweighting.

Yet another practical concern is whether the outcome we observe truly reflects the appropriate notion of benefit/disbenefit. Otherwise, CDTEs similarly may not reflect the right risk measure, and our estimators could further propagate undesirable biases and choices encoded in the data \citep{passi2019problem}. 

Lastly, there may be issues of calibrating risk profiles to human preferences, so as to make CDTEs informative for decision making.
In particular, there may be conflicting risk preferences between the decision making entity (\eg, doctor, policymaker, product manager) and the individual (\eg, patient, citizen, platform user).
If the risk measure employed is in conflict with individual preferences, it is possible to have a negative impact on individuals from their own perspective.
Nonetheless, using the outputs of our models using different risk measures on an individualized basis could potentially be used to make policy decisions for individuals based on their own risk preferences.

\end{document}